\documentclass{article}

\newcommand{\popa}{{\mathscr{A}}}

\newcommand{\popx}{{\mathscr{X}}}

\newcommand{\popd}{{\mathscr{D}}}

\newcommand{\popo}{{\mathscr{O}}}

\newcommand{\popp}{{\mathscr{P}}}

\newcommand{\rot}{{\mathscr{R}}}

\newcommand{\popl}{{\mathscr{L}}}

\newcommand{\R}{\mathbb{R}}

\newcommand\B{\rule[-1.4ex]{0pt}{0pt}}

\newcommand{\bP}{\ensuremath{\mathbb{P}}}
\newcommand{\E}{\mathbb{E}}

\RequirePackage{amsmath,amsfonts,amsthm,natbib}

\usepackage{geometry}                
\usepackage{graphics}
\usepackage{graphicx}
\usepackage{amssymb}
\usepackage{mathrsfs}
\usepackage[T1]{fontenc}
\usepackage[utf8]{inputenc}
\usepackage{authblk}

\newtheorem{theorem}{Theorem}[section]
\newtheorem{definition}{Definition}
\newtheorem{corollary}{Corollary}[section]
\newtheorem{lemma}{Lemma}[section]
\newtheorem{prop}{Proposition}[section]

\usepackage[final]{pdfpages}

\title{Co-clustering for directed graphs: the Stochastic co-Blockmodel and spectral algorithm Di-Sim}
\author[1]{Karl Rohe \thanks{karlrohe at stat dot wisc dot edu}}
\author[1]{Tai Qin} 
\author[2]{Bin Yu}
\affil[1]{Department of Statistics, University of Wisconsin Madison}
\affil[2]{Department of Statistics, University of California, Berkeley}

\begin{document}
\maketitle

\begin{abstract}
%
%
Directed graphs have asymmetric connections, yet the current graph clustering methodologies cannot identify the potentially global structure of these asymmetries.  
 We give a spectral algorithm called \textsc{di-sim} that builds on a dual measure of similarity that correspond to how a node (i) sends and (ii) receives edges.  Using \textsc{di-sim}, we analyze the global asymmetries in the networks of Enron emails, political blogs, and the \textit{c elegans} neural connectome.  In each example,  a small subset of nodes have persistent asymmetries; these nodes send edges with one cluster, but receive edges with another cluster.  Previous approaches would have assigned these asymmetric nodes to only one cluster, failing to identify their sending/receiving asymmetries. 
%
%
%

\hspace{.2in} Regularization and ``projection" are two steps of \textsc{di-sim} that are essential for spectral clustering algorithms to work in practice.  The theoretical results show that these steps make the algorithm weakly consistent under the degree corrected Stochastic co-Blockmodel, a model that generalizes the Stochastic Blockmodel to allow for both (i) degree heterogeneity and (ii) the global asymmetries that we intend to detect.  The theoretical results make no assumptions on the smallest degree nodes. Instead, the theorem requires that the average degree grows sufficiently fast and that the weak consistency only applies to the subset of the nodes with sufficiently large leverage scores.  The results results also apply to bipartite graphs. 

\end{abstract}


\section*{Introduction}



The network analysis literature has primarily studied networks with symmetric relationships (i.e. undirected edges).  However, many networks contain asymmetric relationships (i.e. directed  edges).  
 For example, in a communication network, one person calls the other person.  Citation networks, web graphs, and internet networks are also characterized by asymmetric relationships.  Even networks that are often represented as undirected networks of symmetric edges (e.g. Facebook friendships and road networks) are simplifications of an underlying directed network; in Facebook, each friendship is proposed by one of the friends and received by the other friend.  This induces an asymmetry.  In road networks, one is often interested in the flow of traffic;  anyone who has a reverse commute can confirm that traffic flows asymmetrically.  In biochemical cellular networks, a relationship represents the flow of information and/or energy in the cell.  These are causal graphs, and causality requires  direction.  In these examples and in a wide range of other applications, directed networks more accurately represent the underlying data generating mechanism.  


To model the clustering structure of an undirected network, the Stochastic Blockmodel
assigns each actor to one of $k$ blocks and actors in the same block are exchangeable or ``stochastically equivalent" (\cite{white1976social,holland1983stochastic}).  Specifically, $i$ and $j$ are stochastically equivalent if
\[P(\mbox{$i$ connects to $\ell$}) = P(\mbox{$j$ connects to $\ell$}) \  \mbox{ for every actor $\ell$ in the network}.\] 
This paper extends the notion of stochastic equivalence to directed networks in a way that allows for two separate notions of equivalence, ``stochastically equivalent senders" and ``stochastically equivalent receivers".

To estimate these dual notions of equivalence,  we propose \textit{co}-clustering in the context of networks. In a more general setting, \cite{hartigan1972direct} first proposed co-clustering to simultaneously cluster the rows and columns of a data matrix.  In the network setting, the data matrix could be the adjacency matrix or some form of the graph Laplacian; row $i$ gives the sending pattern for actor $i$ and column $j$ gives the receiving pattern for actor $j$.  Co-clustering for network data is particularly interesting because the rows and columns index the same set of nodes. The Stochastic co-Blockmodel (proposed in Section \ref{model}) clarifies the relationship between  co-clustering and the dual notions of stochastic equivalence.  In this model, there are two separate partitions of the actors; one partition for ``stochastically equivalent senders" and the other partition for ``stochastically equivalent receivers".

In addition to providing a novel framework to understanding the asymmetries in a directed graph, the contributions of this paper are threefold.  The first contribution is algorithmic;  Section \ref{sec:coclust} presents \textsc{di-sim}, a novel and computationally tractable spectral algorithm for co-clustering sparse and heterogeneous data matrices.  While spectral algorithms have become popular in the methodological and theoretical literature, their empirical performance is often less than satisfactory. 
The second contribution is methodological and empirical;  Section \ref{Examples} demonstrates that \textsc{di-sim} performs well on three empirical networks.  In each of the three networks,
 \textsc{di-sim} finds a small subset of nodes  with persistent asymmetries,  sending edges to one cluster and receive edges from another cluster.    
The final contribution is theoretical.  To illustrate the types of asymmetries that \textsc{di-sim} estimates, Section \ref{model} proposes the degree corrected Stochastic co-Blockmodel.  Using this model, Theorem \ref{clusterTheorem} gives conditions under which \textsc{di-sim} misclusters a vanishing fraction of the nodes.

The theoretical results are novel in several ways. 
First, Theorem \ref{clusterTheorem} gives the first statistical estimation results results for directed graphs or bipartite graphs with general degree distributions. 
Second, because \textsc{di-sim} uses the leading singular vectors of a sparse and asymmetric matrix, the proof required novel extensions of previous proof techniques.  These techniques extend the spectral results to bipartite graphs; previous results for bipartite graphs have only studied computationally intractable techniques, e.g. \cite{flynn2012consistent, choi}.  
Third,  the main theorem does not presume that the number of sending clusters is equal to the number of receiving clusters and the theoretical results highlight the difficulties presented when they are not equal. 
Fourth, because we study a sparse degree corrected model, the theoretical results highlight the importance of the regularization and projection steps in \textsc{di-sim}.   
Finally, the results do not depend on the minimum node degree. Instead, the weakly connected nodes affect the conclusions through their statistical leverage scores in the observed graph Laplacian.  From the perspective of numerical linear algebra, the leverage scores are essential to controlling the algorithmic difficulty of computing the singular vectors \citep{mahoney2011randomized}.

\section{Co-clustering} \label{sec:coclust}
Co-clustering (a.k.a. bi-clustering) was first proposed in \citet{hartigan1972direct} for data arranged in a matrix $M \in \R^{n \times d}$.  In addition to clustering the rows of $M$ into $k_r$ clusters, co-clustering simultaneously clusters the columns of $M$ into $k_c$ clusters.   In the past decade, co-clustering has  become an important data analytic technique in biological applications (e.g.  \citet{madeira2004biclustering}, \citet{tanay2004revealing}, \citet{tanay2005biclustering}, \citet{madeira2010identification}), text processing (e.g. \citet{dhillon2001co}, \citet{bisson2008chi}), and natural language processing (e.g. \citet{freitag2004trained}, \citet{rohwer2004towards}). In these settings, \citet{banerjee2004generalized} describes how co-clustering dramatically reduces the number of parameters that one needs to estimate.  This leads to three advantages over traditional clustering: (1) more interpretable results, (2) faster computation, and (3) implicit statistical regularization.


Previous applications of co-clustering have involved matrices where the rows and columns index different sets of objects.  For example, in text processing, the rows correspond to documents, and the columns correspond to words.  Element $i,j$ of this matrix denotes how many times word $j$ appears in document $i$.  The row clusters correspond to clusters of similar documents and the column clusters correspond to clusters of similar words.  In contrast, this paper applies co-clustering to a matrix where the rows and columns index the same set of nodes.  The $i$th row of the matrix identifies the \textit{outgoing} edges for node $i$; two nodes are in the same row cluster if they send edges to several of the same nodes.  The $i$th column of this matrix identifies the \textit{incoming} edges for node $i$; two nodes are in the same row co-cluster if they send edges to several of the same nodes.  As such, each node $i$ is in two types of clusters (one for the $i$th column and one for the $i$th row).  Comparing these two distinct partitions of the nodes can lead to novel insights when compared to the standard co-clustering applications where the rows and columns index different sets. The three examples in Section \ref{Examples} will illustrate how this duality can lead to novel interpretations.

This paper proposes and studies  a spectral co-clustering algorithm called \textsc{di-sim}.  Building on previous spectral co-clustering algorithms (e.g. \citet{dhillon2001co}), \textsc{di-sim} incorporates regularization and projection steps.  These two steps are  essential when there is a large amounts of degree heterogeneity and several weakly connected nodes.   The name \textsc{di-sim} has three meanings.  First, because \textsc{di-sim} co-clusters the nodes, it uses two distinct (but related) similarity measures between nodes: ``the number of common parents" and ``the number of common offspring" to create two different partitions of the nodes.  In this sense,  \textsc{di-sim} means two similarities and two partitions.   
Second,  \textsc{di-}  denotes that this algorithm is specifically for \textit{di}rected graphs.  Finally, \textsc{di-sim}, pronounced ``dice `em", dices data into clusters.

\subsection{DI-SIM; a co-clustering algorithm for directed graphs} \label{alg}
Spectral graph algorithms have a rich history in mathematics, computer science, and statistics (e.g \citet{fiedler1973algebraic, chung1997spectral,Koltchinskii2000}) and this line of literature motivates the \textsc{di-sim} algorithm.  For a detailed account of spectral clustering for undirected graphs, see \citet{vonluxburg2007tsc}.  

The essential algorithmic difference between standard spectral clustering  and \textsc{di-sim} is that \textsc{di-sim} uses the singular value decomposition (SVD) instead of  the eigendecomposition. Previously, \citet{dhillon2001co} proposed using the SVD to co-cluster bipartite networks.  More generally, in the age of big data, SVD has become a canonical algorithm for low rank approximations because there are computationally fast implementations and it generalizes the eigendecomposition to general rectangular matrices. It is defined as follows. 
\begin{definition} \label{svddef}
The \textbf{singular value decomposition} (SVD) factorizes a matrix $M \in \R^{n \times d}$ ($n\ge d$) into the product of orthonormal matrices $U\in \R^{n \times d} , V \in \R^{d \times d}$ and a diagonal matrix $\Sigma \in \R^{d \times d}$ with nonnegative entries,
\[M = U\Sigma V^T.\]
\end{definition}
The columns of $U$ contain the \textbf{left singular vectors}.  The columns of $V$ contain the \textbf{right singular vectors}. The diagonal of $\Sigma$ contains the singular values.  If the matrix $M$ is square and symmetric, then the SVD is equivalent to the eigendecomposition and $U = V$.   In this way, SVD is a generalization of the eigendecomposition. Section \ref{sec:svd} (at the end of this paper) briefly highlights the previous algorithms that employ SVD to explore the structure of graphs.

\subsection{The \textsc{di-sim} algorithm}
Let $G = (V,E)$ denote a graph, where $V$ is a vertex set and $E$ is an edge set. The vertex set $V = \{1, \dots , n\}$ contains vertices or nodes. These are the actors in the graph. This paper considers unweighted, directed edges.  So, the edge set $E$ contains a pair $(i,j)$ if there is an edge, or relationship, from node $i$ to node $j$: $i \rightarrow j$.   The graph can  be represented as an adjacency matrix $A \in \{0,1\}^{n \times n}$:
\[
A_{ij} = \left\{\begin{array}{cl}1 & \textrm{ if  $(i,j)$ is in the edge set}  \\
0 &  \textrm{ otherwise.}\end{array}\right.
\]
If the adjacency matrix is symmetric,  then the graph is undirected.  We are interested in exploring the asymmetries in $A$.  

The graph Laplacian is a function of the adjacency matrix. It is fundamental to spectral graph theory and the spectral clustering algorithm (\cite{chung1997spectral, vonluxburg2007tsc}).  Several previous papers have proposed and or studied various ways of regularizing the graph Laplacian; these regularization steps improve the statistical performance of various spectral algorithms (\cite{page1999pagerank, andersen2006local, chaudhuri2012spectral, amini2013pseudo, qin, joseph2013impact}).  This paper generalizes the regularization proposed in \citet{chaudhuri2012spectral} to directed graphs.  Define the regularized graph Laplacian $L \in \R^{n\times n}$ for directed graphs with the diagonal matrices $P \in \R^{n \times n}$ and $O \in \R^{n \times n}$, regularization parameter $\tau \ge 0$, and identity matrix $I \in R^{n \times n}$,
\begin{equation}
	\label{Ldef}
    \begin{array}{lll}
P_{jj} & = & \sum_k A_{kj}\B  = \sum_k \textbf{1}\{ k \rightarrow j\} \ \mbox{ and } \ P^\tau  = P + \tau I; \\   
O_{ii} & = & \sum_k A_{ik}\B = \sum_k \textbf{1}\{i \rightarrow k\}  \ \mbox{ and } \ O^\tau = O + \tau I; \ \mbox{and} \\   
L_{ij} &=& \frac{A_{ij}}{\sqrt{O^\tau_{ii} P^\tau_{jj}}} = \frac{\textbf{1}\{i \rightarrow j\}}{\sqrt{O^\tau_{ii} P^\tau_{jj}}}  = [(O^\tau)^{-1/2}A(P^\tau)^{-1/2}]_{ij}.
\end{array}
\end{equation}
$P_{jj}$ is the number of nodes that send an edge to node $j$, or the number of parents to node $j$.  Similarly,  $O_{ii}$ is the number of nodes to which $i$ sends an edge, or the number of offspring to node $i$. A more standard definition of the graph Laplacian is $I - O^{-1/2}AO^{-1/2}$.   Our definition also uses $P$ in the normalization and it does not contain $I-$.  These changes are essential to our theoretical results and many of the interpretations of \textsc{di-sim} would not hold otherwise.  The regularized degree matrices, $P^\tau$ and $O^\tau$, artificially inflate every degree by a constant $\tau$.  In the setting of undirected graphs, \citet{qin} showed that in order to make the asymptotic bounds informative, $\tau$ should grow proportionally to the average node degree,  $\sum_i O_{ii}/n$. Note that $\sum_i O_{ii}/n = \sum_j P_{jj}/n$ since the out degree equals to the in degree.  We use the average node degree as the default value for $\tau$.  

To apply \textsc{di-sim} to a bipartite graph on disjoint sets of vertices $U$ and $V$ (e.g. $U$ contains words and $V$ contains documents), let $U$ index the rows of $A$ and $V$ index the columns of $A$.  As such, $A$ is rectangular and $A_{ij} = 1$ if and only if $i \in U$ shares an edge with $j \in V$ (e.g. word $i$ is contained in document $j$).  While the dimensions of $O, P,$ and $L$ must change to reflect that $A$ is rectangular, the definitions in Equations \eqref{Ldef}  remain the same. 

Throughout, for $x \in \R^d$, $\|x\|_2 = \sqrt{\sum_{i = 1}^d x_i^2}$, for $M \in \R^{d \times p}$, $\|M\|$ denotes the spectral norm and $\|M\|_F$ denotes the Frobenius norm.   With the above notation, \textsc{di-sim} is defined as follows.


\vspace{.2 in}
\noindent
\framebox[5.9 in][c]{  \begin{minipage}[l]{5.7in}
{\textsc{di-sim}}
\vspace{.05in}

Input: Adjacency matrix $A \in \{0,1\}^{n \times n}$, regularizer $\tau \ge 0$ (Default: $\tau = $ average node degree), number of row-clusters $k_y$, number of column-clusters $k_z$.
\begin{enumerate}
\item[(1)] Compute the regularized graph Laplacian $L = (O^\tau)^{-1/2}A(P^\tau)^{-1/2}$.
\item[(2)]  Compute the top $K$ left and right singular vectors $X_L \in \R^{n \times K}, X_R \in \R^{n \times K}$, where $K = \min \{k_y,k_z\}.$
\item[(3)] Normalize each row of $X_L$ and $X_R$ to have unit length. That is, define $X_L^* \in \R^{n \times K}, X_R^* \in \R^{n \times K}$, such that 
\[[X_L^*]_i = \frac{[X_L]_i}{\|[X_L]_i\|_2}, \  [X_R^*]_j = \frac{[X_R]_j}{\|[X_R]_j\|_2},\]
where $[X_L]_i$ is the $i$th row of $X_L$ and similarly for $[X_L^*]_i, [X_R]_j, [X_R^*]_j$.
\item[(4)] Cluster the rows of $X_L^*$ into $k_r$ clusters with $(1+\alpha)$-approximate $k$-means  (\citet{kumar2004simple}). Because each row of $X_L^*$ corresponds to a node's sending pattern in the graph, the results  cluster the nodes' sending patterns.
\item[(5)] Cluster the receiving patterns by performing step (4) on the matrix $X_R^*$ with $k_z$ clusters. 
\end{enumerate}
{Output: The clusters from step (4) and (5). }
 \end{minipage}
}

\vspace{.2 in}

When $A$ is undirected, then the left and right singular vectors of $L$ are equal to each other and equal to the eigenvectors of $L$.  In this special case, \textsc{di-sim} is equivalent to previous versions of undirected spectral clustering (e.g. see \cite{vonluxburg2007tsc}, \cite{qin}).

 \subsection{Interpreting the singular vectors}
This subsection examines how the singular vectors of $L$ correspond to the following dual measures of similarity: ``number of common parents" and ``number of common offspring".  Recall that the SVD expresses a matrix $M\in \R^{n \times d}$  as the product of three matrices, $M = U\Sigma V^T$.  Lemma \ref{svdlem} shows how to compute the matrices $U, V$, and $\Sigma$, giving insight into the similarity measures used by \textsc{di-sim}.  The lemma follows from Lemma 7.3.1 in \citet{horn2005matrix}.

\begin{lemma}\label{svdlem}  With SVD, $M = U\Sigma V^T$ for $M\in \R^{n \times d}$. The matrices $U$ and $V$ contain the eigenvectors to the symmetric, positive semi-definite matrices $MM^T$ and $M^TM$ respectively.  Both $MM^T$ and $M^TM$ have the same eigenvalues and these values are contained in the diagonal of $\Sigma^2$.  If $M$ is symmetric, $U$ contains the eigenvectors of $M$ and $U=V$.  
\end{lemma}

This implies that \textsc{di-sim} uses the eigenvectors of two matrices, $L^TL$ and $LL^T$.  To understand these matrices, first look at $A^TA$ and $AA^T$.  
\begin{eqnarray*}
(A^TA)_{ab} &=&  \sum_x \textbf{1}\{x \rightarrow a \mbox{ and } x \rightarrow b\}: \ \ \mbox{The number of common ``parents".} \label{ata}\\
(AA^T)_{ab} &=& \sum_x \textbf{1}\{a \rightarrow x \mbox{ and } b \rightarrow x\}: \ \ \mbox{The number of common ``offspring".}  \label{aat}
\end{eqnarray*}
These two similarity matrices are symmetric and easily interpretable.  
$LL^T$ and $L^TL$ perform a similar task while down-weighting the contribution of high degree nodes and utilizing the regularization parameter $\tau$.


\section{Applications where asymmetric relationships allow for novel insights} \label{Examples}

The next three subsections use \textsc{di-sim} to examine the asymmetries in (i) the email communication network at Enron; (ii) the network of hyperlinks among a set of political blogs; and (iii) the neural connectome of a primitive worm, \textit{c elegans}.  These examples demonstrate how one can leverage the graph asymmetries to make novel insights into the  graph structure.  The examples also demonstrate simple modifications of \textsc{di-sim} that are appropriate in various settings. 

\subsection{Detecting malfeasance at Enron} \label{enron}

The defunct corporation Enron went bankrupt on December 2, 2001 because ``its reported financial condition was sustained substantially by an institutionalized, systematic, and creatively planned accounting fraud" (\citet{wikipedia}).  This section  examines a communication network formed with a portion of the corporations' emails that were made publicly available as a result of the federal investigation into corporate misconduct. We use  \textsc{di-sim} to search for ``bottleneck" communicators, or people that relayed information from one part of the organization to another.

The emails used in the following analysis form a communication network for 154 employees of Enron between 1998 and 2002 (\cite{cohen}).  In our analysis, we set $A_{ij}$ as the number of emails that $i$ sends to $j$ over the entire time period.  This is a weighted network.  While the data set also provides the text of the emails,  we  only use the ``metadata", i.e. the network $A$.  

\begin{figure}[h] 
   \centering
	\includegraphics[width = 6in]{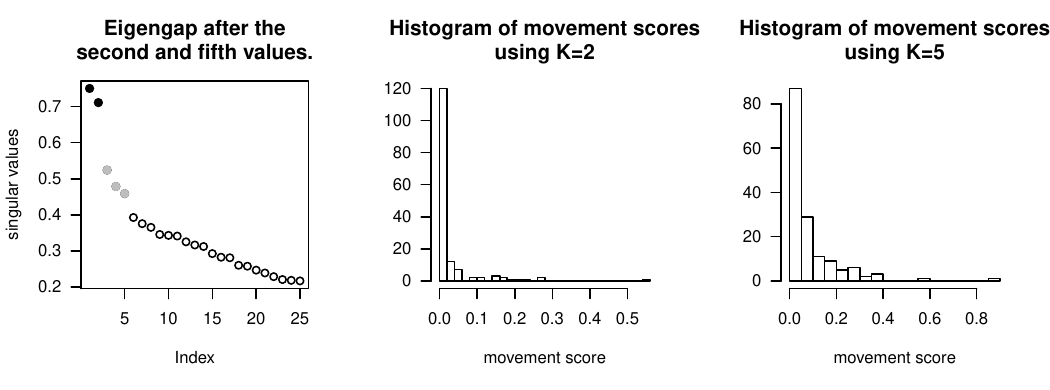}
   \caption{The left panel displays the top 25 singular values of $L$.  There are two eigengaps.  The first eigengap, suggests $K = 2$ using the singular values in solid black.  The second eigengap, suggests $K = 5$ by adding the singular values in solid grey.  Using $K =2$, the center panel gives a histogram of the movement scores $m_i^{(K)}$ as defined in Equation \ref{move}. The right panel gives a histogram of the movement scores using $K = 5$.  Each histogram has an outlier.  For $K =2$, the outlier is Enron's Director for Regulatory and Government Affairs Jeff Dasovich.  For $K=5$, the outlier is Bill Williams who is discussed in the text below.}
   \label{fig:enron}
\end{figure}

\subsubsection{Data Analysis} 
This section does not use the full \textsc{di-sim} algorithm; there are two simplifications.
\begin{enumerate}
\item The rows of the singular vector matrices $X_L, X_R \in R^{1222 \times K}$ are not projected onto the unit sphere.  That is, step (c) in \textsc{di-sim} is skipped.
\item Instead of running k-means,  the asymmetry in the graph is investigated  by directly comparing the rows of $X_L$ and $X_R$ with the \textbf{movement score}, defined as 
\begin{equation} \label{move}
m_i^{(K)} = \left( \sum_{\ell = 1}^K ([X_L]_{i\ell} - [X_R]_{i \ell})^2\right)^{1/2}.
\end{equation}
If one were to ignore edge direction by symmetrizing the network, then $m_i^{(K)}$ would be zero for all $i$.  As such, it measures the asymmetry in a node's connections.  
\end{enumerate}

The left panel in Figure \ref{fig:enron} displays the top 25 singular values of $L$.  The center panel gives the histogram of the movement scores $m_i^{(2)}$.  The right panel gives the histogram of $m_i^{(5)}$. The outlier for $K =2$ is  Enron's Director for Regulatory and Government Affairs Jeff Dasovich.  Using $K = 5$, the outlier is an energy trader at Enron named Bill Williams.  
%
%

The large movement scores for Dasovich and Williams could be due to three possibilities.  First, they could  receive information from one part of the network and transmits it to another part of the network (i.e. act as a bottleneck communicator); second, they could take information and not relay that information; or third, they could receive little information through this communication network, but transmit lots of information.  In fact, in the weighted network ($A_{ij}$ is number of emails from $i$ to $j$), Dasovich has the largest out-degree and Williams' has the 10th largest out-degree.  Dasovich has the ninth largest in-degree and Williams has the 45th highest in-degree (out of $n = 154$).  This rules out the second and third possibilities, suggesting that both Dasovich and Williams are bottleneck communicators.

Although such network patterns do not necessarily imply criminal activity, the analysis identifies Enron employee Bill Williams as a clear outlier.  Using qualitative evidence not associated with the methods presented here, Williams was convicted of creating artificial energy shortages by ordering power plants to temporarily shut down.  The New York Times reported on the incident and quoted from audio recordings of Bill Williams telling a power plant to shut down. The day after that audio recording, roughly half a million Californians suffered from rolling blackouts (\cite{nytimes}). 

%
%
%
%
%

Although Williams' communications with the power plant make him a bottleneck communicator, it is worth noting that our vertex set in this data does not contain people outside of Enron.  As such, Williams was identified for playing the bottleneck communicator for other activities within Enron.   Importantly, this analysis would have been infeasible if we had ignored edge direction.  The data in this section have been extensively preprocessed by \citet{zhou2007strategies} and \citet{perry2010point}.

\subsection{Blog network during the 2004 US presidential election} \label{blog}
 
In the 2004 US presidential election, political blogs contributed to the election media landscape for the first time.  In order to better understand the role of these blogs, \cite{Adamic} recorded the hyperlink connections among these blogs and found that the connections between blogs were highly related to the blog's political persuasion.  In subsequent research, \cite{karrer2011stochastic, chen2012fitting}, and  \cite{zhao2012consistency} estimated the political partition from the network alone.  In contrast to the work presented here, each of these previous analyses symmetrized the edge directions.  As such, they found a single partition of the blogs into conservative and liberal blogs.  However, co-clustering with \textsc{di-sim} finds two partitions, one based on how the blogs send hyperlinks and another based on how the blogs receive hyperlinks.  These two partitions are roughly similar.\footnote{Both partitions roughly align with the political divide of liberal vs.\ conservative blogs.}  This suggests that most blogs with similar sending patterns have similar receiving patterns. However,  some blogs send hyperlinks to conservative blogs and receive hyperlinks from liberal blogs, and vice versa.  For these blogs, the direction of the edges is particularly salient.   Using \textsc{di-sim}, we seek to identify and characterize these blogs.  Figure \ref{bottleFig} illustrates how a node could belong to opposite sending and receiving clusters.  


\begin{figure}[htbp] 
   \centering
   \includegraphics[width=6in]{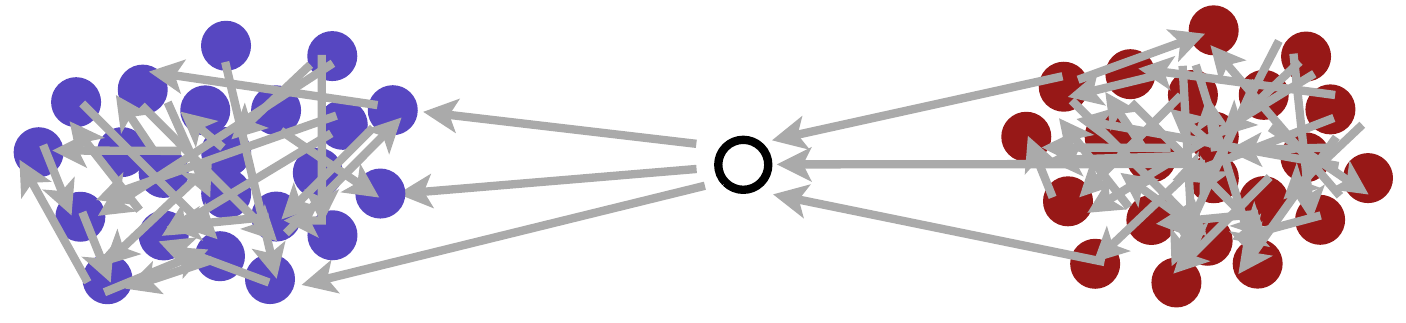} 
   \caption{In this diagram, there are two clusters and a bottleneck node between the two clusters.  In the sending cluster, this node joins the nodes on the left.  In the receiving cluster, this node joins the nodes on the right. }
   \label{bottleFig}
\end{figure}




\subsubsection{Data description}
To create the network, \cite{Adamic} curated  a list of the top 1,494 political blogs and, in February of 2005, (a)  recorded the front page of each blog and (b) identified the hyperlinks that point to other blogs on the list.  From these links, \cite{Adamic} created a directed network.\footnote[1]{See \cite{Adamic} for a more complete description of how the list of 1,494 blogs was curated.}  Each blog was identified as liberal or conservative.  Some of these labels were manually identified and some of the labels are self-reported to one of several blog directories. While these labels may be subject to various types of errors, they are generally consistent with the network connectivity and the names of the blogs (e.g. xtremerightwing.net vs. loveamericahatebush.com).  We will thus refer to these labels as the true labels.  To refer to the blogs on either side of the political partition, we will use the terms $\{$Kerry, dem, liberal$\}$ interchangeably and the terms $\{$Bush, gop, conservative$\}$  interchangeably.

We restrict our analysis to the 1,222 blogs in the largest connected component; this removes 266 nodes with no edges and two blogs linked together without any connections to the largest component.  Our analysis concerns the 586 liberal blogs and 636 conservative blogs that remain.  While this network is sparse (the average degree is 16)  clustering is feasible because there are roughly 10 times as many edges between blogs of the same party than between blogs of different party affiliations \citep{Adamic}.

Because there are two political parties, we set $k=2$ for both the sending and receiving  clusters.  This section makes one modification to \textsc{di-sim}.  Instead of running k-means twice (once for $X_L \in \R^{1222 \times 2}$ and once for $X_R\in \R^{1222 \times 2}$), the analysis runs k-means on $X_L$ and $X_R$ simultaneously.  That is, $X_L$ and $X_R$ are stacked into a single tall matrix in $\R^{2444 \times 2}$ and k-means is run on the rows of this tall matrix.  This makes the labels of the left and right clusters comparable.  After running k-means on the 1,222 blogs in the largest connected component, subsequent analysis is restricted to 
the blogs that have at least three incoming edges and at least three outgoing edges.  There are 549 such blogs. Of these blogs, 543 are clustered into sending and receiving  clusters that are nearly identical; this partition broadly agrees with the true labels of ``Kerry" and ``Bush" blogs. 

While 543 of the 549 blogs are clustered into identical sending and receiving clusters, the remaining six are clustered into different sending and receiving clusters (See Table 1).  Five of these blogs are ``dem2gop" blogs that appear to take links from Kerry (i.e. dem) blogs and send links to Bush (i.e. gop) blogs. The final blog in the table (quando.net) is the only ``gop2dem" blog, taking more edges from Bush blogs and sending links to Kerry blogs. 

All of the six blogs are labeled as Kerry blogs.  
However, we visited the blog urls and performed related web searches (Table 2).  Many of the sites are now defunct.  Interestingly, the only gop2dem blog in the entire analysis, quando.net, appears mislabeled in the original data set.  \citet{Adamic} label it as a Kerry blog.  Upon closer inspection, this blog hosts a collection of conservative/libertarian bloggers (e.g. ``Face it - the only thing Bush can brag about is his comparative conservative advantage over Kerry. And that's akin to saying a tornado is--comparatively--better at home improvement projects than a hurricane.").  

This analysis finds six blogs with asymmetric community memberships.  Each of these six blogs appear to be doing ``opposition research," where they link to blogs that hold different political views.  As such, asymmetric blogs link to content that they dislike. We found no evidence of any asymmetric blogs receiving links from the opposite party. This suggests that the incoming edges appear to be more informative for detecting the community membership of a political blog.  This analysis is only feasible because \textsc{di-sim} respects the asymmetry between incoming and outgoing edges.


\begin{table} 
\caption{Of the 549 blogs that have at least three incoming edges and at least three outgoing edges, these are the only six blogs whose receiving cluster (from.cluster) is different from their sending cluster (to.cluster).  The numbers to the right of the line are generated using the labels provided in the data set.  These numbers reveal that di-sim identifies the nodes with asymmetric relationships between the true blocks.}

\begin{tabular}{lc|llll}
  \hline
 blog url & from.clust 2 to.clust & from.dem & from.gop & to.dem & to.gop \\ 
  \hline
chepooka.com & dem2gop & 13 & 2 & 1 & 2 \\ 
clarified.blogspot.com & dem2gop & 4 & 2 & 0 & 6 \\ 
politics.feedster.com & dem2gop & 3 & 0 & 13 & 18 \\ 
polstate.com & dem2gop & 31 & 7 & 3 & 2 \\ 
shininglight.us & dem2gop & 2 & 2 & 4 & 7 \\ 
qando.net & gop2dem & 5 & 57 & 14 & 10 \\ 
   \hline
\end{tabular}

\end{table}

\begin{table} 
\caption{After accounting for the fact that the data set appears to mislabel quando.net as a liberal blog, all asymmetric blogs link to blogs of the opposite political leaning. }
\begin{tabular}{l|ll}
  \hline
 blog url & label in data set  & upon visit\\ 
  \hline
chepooka.com & liberal & unclear, possibly defunct\\
clarified.blogspot.com & liberal & Kerry supporter\\
politics.feedster.com & liberal & defunct, evidence for Kerry supporter\\
polstate.com & liberal & defunct, old twitter feed self-identifies as ``pan-partisian"\\
shininglight.us & liberal & defunct \\
qando.net & liberal & collection of \textit{conservative} bloggers,  \\
&& \ \ see \verb"http://www.qando.net/archives/2004_09.htm"\\
   \hline
\end{tabular}

\end{table}


\subsection{The neural connectome of \textit{c elegans}}  \label{worm}

This section examines the neural connectome of the male \textit{Caenorhabditis elegans} (\textit{c elegans}), a 1mm long worm.  The chemical connections between the neurons of \textit{c elegant} create a directed network and a directed analysis highlights vast dissimilarities between the sending and receiving patterns.  Said another way, $X_L$ represents a different structure than $X_R$.  For some neurons, this reflects previously understood behavior that is relevant to the understanding of the connectome (e.g. see Figure 6 in \cite{jarrell2012connectome} for a discussion of the role of PVV neurons in feedforward loops).  For other neurons, our analysis suggests areas for future inquiry.  
%
%

\subsubsection{Data description} \textit{c elegans} is well suited to laboratory research--it is easy to store and reproduce because it is 1mm in length, it is easy to witness an organism's state because its exterior is transparent, and the anatomy is easy to identify and catalogue because the adult male is composed of exactly 1031 cells, of which exactly 383 are neurons.  As a result, \textit{c elegans} has become a model organism for several areas of biological research, including neurology.  For example, it was the first organism with a fully sequenced genome and also the first with a complete wiring diagram of the neurological connections (\cite{white1986structure}). 

Our analysis concerns the chemical connections in the connectome of the male \textit{c elegans}. \citet{jarrell2012connectome} mapped the posterior neural connectome of the male \textit{c elegans} by slicing the posterior of the 1mm long worm into a series of 5,000 serial slices, 70 nm to 90 nm thick.  Each slice was imaged with an electron microscope, and the neurons from each slice were mapped to the neurons in the adjacent slices.  Piecing these mappings together created a three dimensional image of the organism that reveals the synaptic connections between the neurons.  
The construction of the connectome used both computational tools for automatic information extraction and a substantial amount of human judgment.  

   This section investigates the directed graph that encompasses the chemical connections among the neurons, muscles, and gonad.  In the posterior chemical connectome, there are 
   \begin{itemize}
   \item 126 nodes that send at least one edge and receive at least one edge, 
   \item one node that sends at least one edge and receives no edges, and   
   \item  73 nodes that send no edges and receive at least one edge.
   \end{itemize}
  Of the nodes that send at least one edge, the average out degree is 18.  Of the nodes that receive edges, the average in degree is 11.5.  Both of these degree calculations are on the \textit{unweighted} graph.  In fact, each edge has an edge weight that corresponds to the size of the synaptic connection.  The larger connections produce a more robust connection between neurons.  More details can be found in \citet{jarrell2012connectome}.  
The distribution of these edge weights has a long tail.  Based on a preliminary analysis, the edge weights were log-transformed.

The analysis uses \textsc{di-sim} with the default value of $\tau$.  Because the original paper \citet{jarrell2012connectome} estimated seven communities, we also estimate seven communities.  
After normalizing the rows of $X_L \in \R^{127 \times 7}$ and $X_R \in \R^{199 \times 7}$, these matrices are combined into a single, taller matrix in $\R^{326 \times 7}$ and k-means is applied to this matrix with $K = 7$.   By running k-means once (instead of twice), the sending and receiving clusters are more easily comparable because they have the same cluster center.

\subsubsection{Results} 
Using \textsc{di-sim} to estimate the sending and receiving clusters, Figure \ref{fig:B} shows how the co-clusters connect to each other.  It displays the matrix $\hat B$ which is an estimate of the matrix $B$ in Definition \ref{directedBlockmodel}.  $\hat B_{u,v}$ is a proportion.  The denominator is the number of node pairs $(i,j)$ with $i$ in sending cluster $u$ and $j$ in receiving cluster $v$.  The numerator is the number of such pairs that connect.  This matrix has a strong  diagonal which suggests that if an edge comes from sending block $u$, then it probably points to a node in receiving block $u$.
While \textsc{di-sim} was run on the weighted graph, Figure \ref{fig:B} computes $\hat B$ with the unweighted graph.  When $\hat B$ is computed on the weighted graph (i.e. $\hat B_{u,v}$ is average weight of the edges from block $u$ to block $v$), the results are largely unchanged.

\begin{figure}[h] 
   \centering
	\includegraphics[width = 4in]{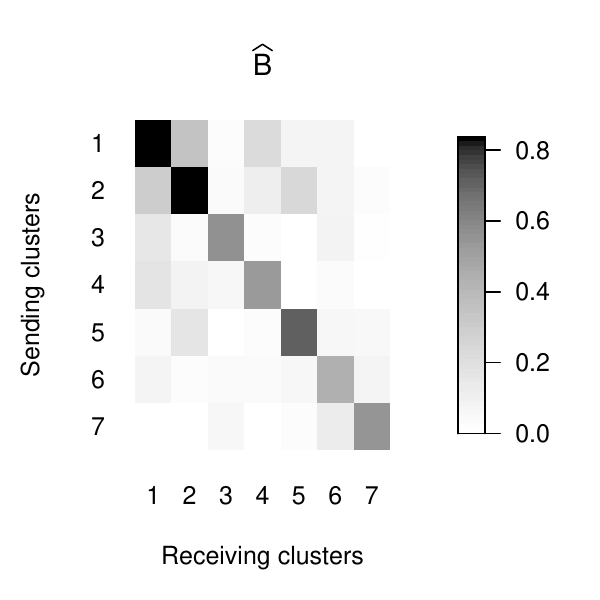}
   \caption{Element $u,v$ is darker when there are more edge from block $u$ to block $v$.      A strong diagonal in this matrix suggests that if an edge comes from a node in sending block $u$, then it probably goes to a node in receiving block $u$. } 
   \label{fig:B}
\end{figure}

Figure \ref{fig:big} presents the left and right partitions of the \textit{c elegans} connectome as estimated by \textsc{di-sim}.  The figure compares the two \textsc{di-sim} partitions with the single partition estimated in the original paper (\cite{jarrell2012connectome}) in which they used the spectral technique of \citet{leicht2008community}.  

\begin{figure}[p] 
   \centering
	\includegraphics[height = 8in]{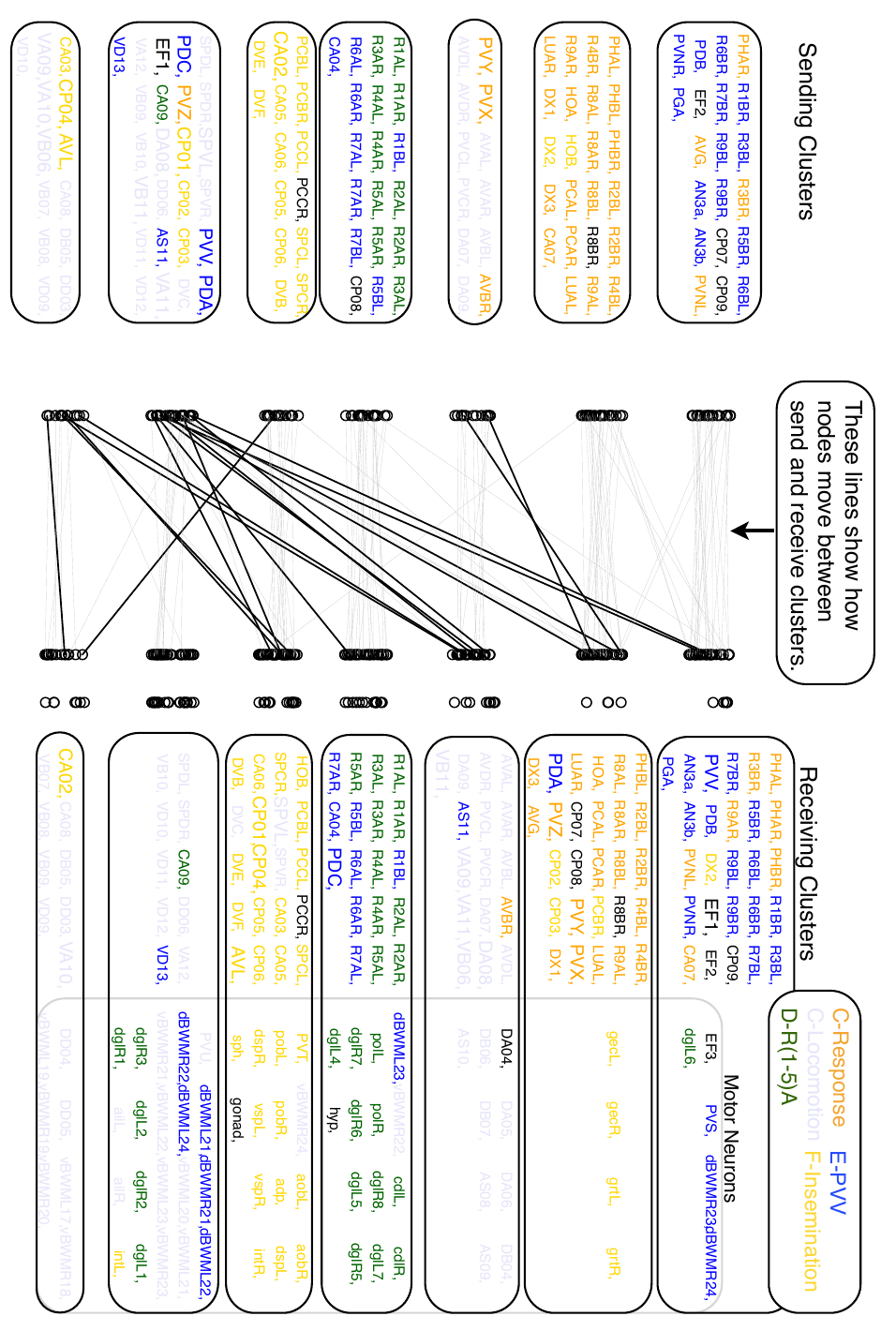}
   \caption{Cluster 1 in Figure \ref{fig:B} corresponds to the top cluster in this figure, cluster 2 corresponds to the second cluster from the top, and so on. }
   \label{fig:big}
\end{figure}

Figure \ref{fig:big} presents three partitions of the nodes.  The first two partitions correspond to the sending clusters (on left) and receiving clusters (on right) in \textsc{di-sim}.\footnote[1]{Some nodes are listed off to the right side of the receiving cluster.  These are nodes that do not send any edges, thus they do not have a sending cluster.  These nodes are largely motor neurons that control muscles.}  Because the k-means step was run only once, the left and right clusters are comparable.  So, the vertical orientation of the clusters is informative because the $i$th sending cluster from the top sends several edges to the $i$th receiving cluster from the top.  

Each neuron has exactly one line that connects the node's sending cluster to the node's receiving cluster.\footnote{The lines in Figure \ref{fig:big} do \textit{not} represent the edges in the graph.}  Darker lines indicate that the neuron moves further between its left and right representations in the singular vectors (as measured by the movement score in Equation \ref{move}).  If there were no co-clustering structure, then all of the lines would be horizontal.  However, several lines traverse diagonally, connecting different clusters, and thus indicating non-trivial co-clustering structure.  To  identify which neurons move a further distance, they are written in a slightly larger font in the sending and receiving clusters.

The final partition represented in Figure \ref{fig:big} is represented by the color of the text; this partition  corresponds to the communities or modules estimated in  \cite{jarrell2012connectome}.  The sensory input to the Response Module (orange) comes from the ventral side of the worm's fan; this module plays an important role in helping the worm physically align with another worm for reproduction.  The response module feeds into the Locomotion Module (pink).  The locomotion module contains the body-wall motor neurons, helping the worm to move. The R(1-5)A module (green) contains sensory neurons that ``promote ventral curling of the tail during mating".  The PVV Module (blue) is likely ``involved in aspects of male posture during mating".  The Insemination Module (yellow) contains neurons that ``will take over the male's behavior once the vulva is sensed".  The interpretation of the clusters in \cite{jarrell2012connectome} comes from that paper.

Figure \ref{fig:big} suggests that there is co-clustering structure beyond the standard one-way clustering.  In particular,  	PVV, PVX, PVY, and PVZ neurons are all in large bold font because they have large movement scores.  These findings are consistent with the discussion of feedforward circuits in \cite{jarrell2012connectome}.  In particular, Figure 6 in \cite{jarrell2012connectome} illustrates how these neurons (PVV, PVX, PVY, and PVZ) send most of their edges to neurons in separate clusters.  


While Figure \ref{fig:B} shows that most edges stay within the same cluster, Figure \ref{fig:big} shows that many of the nodes do not stay within the same cluster.  In particular ten of the 25 nodes in sending cluster 2 are not in receiving cluster 2; they move.  Instead of using the symmetric notion of clustering, this analysis co-clusters the sending and receiving patterns with \textsc{di-sim}, revealing several persistent asymmetries across the network.

\section{Stochastic co-Blockmodel} \label{model}

This section proposes a statistical model for a directed graph with dual notions of stochastic equivalence.  Despite the fact that \textsc{di-sim} is not a model based algorithm, when the graph is sampled from this model, \textsc{di-sim} will estimate these dual partitions.


\subsection{Stochastic equivalence, a model based similarity} \label{modsim}
Stochastic equivalence is a fundamental concept in classical social network analysis. In the Stochastic Blockmodel, two nodes are in the same block if and only if they are stochastically equivalent (\cite{holland1983stochastic}). 
In a directed network, two nodes $a$ and $b$ are stochastically equivalent if and only if both of the following hold:
\begin{eqnarray}
\label{sep} P(a \rightarrow x) &=& P(b \rightarrow x) \ \ \forall x \ \mbox{ and } \\
\label{seo} P(x \rightarrow a) &=& P(x \rightarrow b) \ \ \forall x
\end{eqnarray}
where $a \rightarrow x$ denotes the event that $a$ sends an edge to $x$.   Separating these two notions allows for co-clustering structure. Two nodes $a$ and $b$ are \textit{stochastically equivalent senders} if and only if Equation \ref{sep} holds.  Two nodes $a$ and $b$ are \textit{stochastically equivalent receivers} if and only if Equation \ref{seo} holds.  These two concepts correspond to a model based notion of co-clusters and they are simultaneously represented in the new Stochastic co-Blockmodel.

\subsection{A statistical model of co-clustering in directed graphs}

The Stochastic Blockmodel provides a model for a random network with $K$ well defined blocks, or communities (\cite{holland1983stochastic}).  
%
%
%
The Stochastic co-Blockmodel is an extension of the Stochastic Blockmodel.  

This model naturally generalizes to bi-partite graphs, where the rows and the columns of $A$ index different sets of actors (e.g. words and documents).  As such, the rest of the paper allows for a different number of rows ($N_r$) and columns ($N_c$) in the adjacency matrix $A$.  Using the notation from the previous sections, a directed graph would satisfy $N_r = N_c=n$.

\begin{definition} \label{directedBlockmodel}Define three nonrandom matrices, $Y \in \{0,1\}^{N_r \times k_y}, Z \in \{0,1\}^{N_c \times k_z}$ and $B \in [0,1]^{k_y \times k_z}$.  Each row of $Y$ and each row of $Z$ has exactly one 1 and each column has at least one 1.  Under the \textbf{Stochastic co-Blockmodel} (ScBM), the adjacency matrix $A \in \{0,1\}^{N_r \times N_c}$ is random such that $\E(A) = Y B Z^T$.  Further, each edge is independent, so the probability distribution factors
\[P(A) = \prod_{i, j} P(A_{ij}).\]
\end{definition}

Without loss of generality, we will always presume that $k_y \le k_z$.  

In the Stochastic Blockmodel, $\E(A) = ZBZ^T$.  In the ScBM, $\E(A) = YBZ^T$. In this definition, $Y$ and $Z$ record two types of block membership which correspond to the two types of stochastic equivalence (Equations  \ref{sep} and \ref{seo}).  Denote $y_i$ as the $i$th row of $Y$ and $z_i$ to be the $i$th row of $Z$.
\begin{prop} \label{stocheq}
Under the ScBM for a directed graph,   if $y_i = y_j$, then nodes $i$ and $j$ are stochastically equivalent senders, Equation \ref{sep}.  Similarly, if $z_i = z_j$, then nodes $i$ and $j$ are stochastically equivalent receivers, Equation \ref{seo}.
\end{prop}


\citet{wang1987stochastic} previously proposed and studied a directed Stochastic Blockmodel.  However, our aims are different. 
Where \citet{wang1987stochastic} sought to understand the dependence between $A_{ij}$ and $A_{ji}$, the current paper seeks to understand the co-clustering structure of the blocks.  Importantly, where we use two types of stochastic equivalence (sending and receiving), \citet{wang1987stochastic} uses only one type of stochastic equivalence which implies that if two nodes are stochastically equivalent senders, then the nodes are also stochastically equivalent receivers and vice versa.  
By encoding co-clustering structure, the ScBM more closely aligns with the concept of separately exchangeable arrays (e.g. see \citet{persi} and \citet{choi}).

\subsubsection{Degree correction}

The degree-corrected Stochastic Blockmodel generalizes the Stochastic Blockmodel to allow for nodes in the same block to have highly heterogeneous degrees (\cite{karrer2011stochastic}). 
Theorem \ref{clusterTheorem} below studies a similar generalization of the ScBM.   The Degree-Corrected Stochastic co-Blockmodel (DC-ScBM)
adds two sets of parameters ($\theta^y_i >0, i = 1,...,N_r$ and $\theta^z_j >0, j = 1,...,N_c$) that control the in- and out-degrees for each node. Let ${\bf B}$ be a $k_y \times k_z$ matrix where ${\bf B}_{ab} \geq 0$ for all $a, b$. Then, under the DC-ScBM
\[P(A_{ij} = 1) = \theta^y_i\theta^z_j{\bf B}_{y_iz_j}\]
where $\theta^y_i\theta^z_j{\bf B}_{y_iz_j} \in [0, 1]$.   
Note that parameters $\theta^y_i$ and $\theta^z_j$ are arbitrary to within a multiplicative constant that is absorbed into ${\bf B}$.  
To make it identifiable, we impose the constraint that within each row block, the summation of $\theta^y_i$s is $1$. That is, for each row-block $s$,
\[\sum_i \theta^y_i \textbf{1}(Y_{is} = 1) = 1.\] 
Similarly, for any column-block $t$, we impose 
\[\sum_j \theta^z_j \textbf{1}(Z_{jt} = 1) = 1.\]
Under this constraint, ${\bf B}$ has explicit meaning: ${\bf B_{st}}$ represents the expected number of links from row-block $s$ to column-block $t$.
Under the DC-ScBM, define $\popa \triangleq \E A$.  This matrix can be expressed as a product of the matrices,
\[\popa = \Theta_y Y{\bf B}Z^T \Theta_z,\]
where 
$\Theta_y$ is a diagonal matrix whose $ii$'th element is $\theta^y_i$ and $\Theta_z$ is defined similarly with $\theta^z_j$.

\subsection{Estimating the Stochastic co-Blockmodel with \textsc{di-sim}} \label{asymptotic} 
 
Theorem \ref{clusterTheorem} bounds the number of nodes that \textsc{di-sim} ``misclusters".  This demonstrates that the co-clusters from \textsc{di-sim} estimate both the row- and column-block memberships, one in matrix $Y$ and the other in matrix $Z$, corresponding to the two types of stochastic equivalence.    This implies that the two notions of stochastic equivalence relate to the two sets of singular vectors of $L$.

In a diverse set of large empirical networks, the optimal clusters, as judged by a wide variety of graph cut objective functions, are not very large (\cite{leskovec2008statistical}).  To account for this, the results below limit the growth of community sizes by allowing the number of communities to grow with the number of nodes.  Previously, \citet{rohe, choi2010stochastic, hdsbm}, and \cite{sharmo} have also studied this high dimensional setting for the undirected Stochastic Blockmodel. 

Several previous papers have explored the use of spectral tools to aid the estimation of the Stochastic Blockmodel, including \citet{mcsherry2001spectral, dasgupta2004spectral, coja2009finding, ames2010convex, rohe, sussman2011consistent, chaudhuri2012spectral, joseph2013impact, qin, sarkar2013role, krzakala2013spectral, jinScore}; and \citet{ lei2014consistency}.  
The results below build on this previous literature in several ways.  
Theorem \ref{clusterTheorem} gives the first statistical estimation results for directed graphs or bipartite graphs with general degree distributions. Because we study a graph that is directed,  \textsc{di-sim} uses the leading singular vectors of a sparse and asymmetric matrix.  As such, the proof required novel extensions of previous proof techniques.  These techniques allow the results to also hold for bipartite graphs; previous results for bipartite graphs have only studied computationally intractable techniques, e.g. \cite{flynn2012consistent, choi}.  For directed graphs and particularly for bipartite graphs, it is not necessarily true that the number of sending clusters should equal the number of receiving clusters.  Theorem \ref{clusterTheorem} below does not presume that the number of sending clusters equals the number of receiving clusters; the theoretical results highlight the statistical price that is paid when they are not equal.  Finally, we study a sparse degree corrected model  and the theoretical results highlight the importance of the regularization and projection steps in \textsc{di-sim}.   

Previous theoretical papers that use the non-regularized graph Laplacian all require that the minimum degree grows with the number of nodes (e.g. \citet{rohe, sarkar2013role, lei2014consistency}).  However, in many empirical networks, most nodes have 1, 2, or 3 edges.  In these settings, the non-regularized graph Laplacian often has highly localized eigenvectors that are uninformative for estimating large partitions in the graph. 
Because \textsc{di-sim} uses a \textit{regularized} graph Laplacian, the concentration of the singular vectors does not require a growing minimum node degree.  Several previous papers have realized the benefits of regularizing the graph Laplacian (e.g. \cite{page1999pagerank, andersen2006local, amini2013pseudo, chaudhuri2012spectral, qin, joseph2013impact}).
While the regularized singular vectors concentrate without a growing minimum degree, the weakly connected nodes effect the conclusions through their statistical leverage scores.  From the perspective of numerical linear algebra, the leverage scores and the localization of the singular vectors are essential to controlling the algorithmic difficulty of computing the singular vectors \citep{mahoney2011randomized}.

\subsubsection{Population notation}

Recall that  $\popa = \E(A)$ is the population version of the adjacency matrix $A$.  Under the Degree-Corrected Stochastic co-Blockmodel,
\[\popa = \Theta_y Y {\bf B} Z^T \Theta_z,\]
Similar to Equation (\ref{Ldef}), define regularized population versions of  $O$, $P$, and $L$ as
\begin{equation}
	\label{mainpopdef}
    \begin{array}{lll}
\popo_{jj} = \sum_k \popa_{kj}  \B \\ 
\popp_{ii} = \sum_k \popa_{ik}  \B \\ 
\popo_\tau = \popo+\tau I, \quad\quad \popp_\tau = \popp+\tau I \B \\
\popl  =  \popo_\tau^{-\frac{1}{2}} \popa \popp_\tau^{-\frac{1}{2}} \B
\end{array}
\end{equation}
where $\popo$ and $\popp$ are diagonal matrices.
The population graph Laplacian $\popl$ has an alternative expression in terms of $Y$ and $Z$.



\begin{lemma} \label{lemPopL}(Explicit form for $\popl_\tau$) Under the DC-ScBM with parameters $\{{\bf B}, Y, Z, \Theta_Y, \Theta_Z\}$, define $\Theta_{Y,\tau}\in \R^{N_r \times N_r}(\Theta_{Z,\tau}\in \R^{N_c \times N_c})$ to be diagonal matrix where
 \[[\Theta_{Y,\tau}]_{ii} = \theta^Y_i\frac{\popo_{ii}}{\popo_{ii} + \tau} \quad\quad [\Theta_{Z,\tau}]_{jj} = \theta^Z_j\frac{\popp_{jj}}{\popp_{jj} + \tau}.\]
Then $\popl$ has the following form,
\begin{align*}
 \popl = \popo_\tau^{-\frac{1}{2}}\popa \popp_\tau^{-\frac{1}{2}} = \Theta_{Y,\tau}^{\frac{1}{2}}YB_LZ^T\Theta_{Z,\tau}^{\frac{1}{2}},
\end{align*}
for some matrix $B_L \in \R^{k_y \times k_z}$ that is defined in the proof. 
\end{lemma}
The proof of Lemma \ref{lemPopL} is in Section \ref{populationalg}, in the supplementary materials.

\subsubsection{Definition of misclustered}

Rigorous discussions of clustering require careful attention to  identifiability. In the ScBM, the \textit{order} of the columns of $Y$ and $Z$ are unidentifiable.  This leads to difficulty in defining ``misclustered".  Theorem \ref{clusterTheorem} uses the following definition of misclustered that is extended from \citet{rohe}.

By the singular value decomposition, there exist orthonormal matrices $\popx_L\in \R^{N_r \times k_y}$ and $ \popx_R \in \R^{N_c \times k_y}$ and diagonal matrix $\Lambda \in \R^{k_y \times k_y}$ such that 
\[\popl = \popx_L\Lambda \popx_R^T.\]
Define $\popx^*_L$ and $ \popx^*_R$ as the row normalized population singular vectors,
\[[\popx^*_L]_i = \frac{[\popx_L]_i}{||[\popx_L]_i||_2}, \quad [\popx^*_R]_j = \frac{[\popx_R]_j}{||[\popx_R]_j||_2}.\]
Unless stated otherwise, we will presume without loss of generality that $k_y \le k_z$. If rank$(B) = k_y$, then there exist matrices $\mu^y \in \R^{k_y \times k_y}$  and $\mu^z \in \R^{k_z \times k_y}$ such that  $Y \mu^y = \popx^*_L$ and $Z \mu^z = \popx^*_R$ (implied by Lemma \ref{lemsvd} in the supplementary materials).  Moreover, the rows of $\mu^y$ are distinct; with a slightly stronger assumption, the rows of $\mu^z$ are also distinct.  As such,  k-means applied to the rows of $\popx_L^*$ will reveal the partition in $Y$.  Similarly for  $\mu^z$, $\popx_R^*$, and $Z$.  As such,  \textsc{di-sim} applied to the population Laplacian, $\popl$, can discover the block structure in the matrices $Y$ and $Z$.

Let $X_L \in \R^{N_r \times k_y}$ be a matrix whose orthonormal columns are the right singular vectors corresponding to the largest $k_y$ singular values of $L$. \textsc{di-sim} applies $k$-means (with $k_y$ clusters) to the rows of $X^*_L$, denoted as $u_1, \dots, u_{N_r}$. Each row is assigned to one cluster and each cluster has a centroid.

\begin{definition} \label{centroid}
For $i =1, \dots, N_r$, define $c_i^L \in \R^{k_y}$ to be the centroid corresponding to $u_i$ after running $(1+\alpha)$-approximate k-means on $u_1, \dots, u_{N_r}$ with $k_y$ clusters.
\end{definition}

If $c_i^L$ is closer to some population centroid other than its own, i.e. $y_j \mu^y$ for some $y_j \ne y_i$, then we call node $i$ $Y$-misclustered.  This definition must be slightly complicated by the fact that the coordinates in $X_L$ must first align with the coordinates in $\popx_L$.  So, the definitions below include an additional rotation matrix $\rot_L$.

\begin{definition} 
The set of nodes $Y$-misclustered is
\begin{equation} \label{mydef}
\mathscr{M}_y = \left\{i : \|c_i^L- y_i \mu^y\rot_L\|_2 > \|c_i^L - y_j \mu^y\rot_L\|_2 \ \mbox{ for any } \ y_j \ne y_i\right\},
\end{equation}
where $\rot_L$ is the orthonormal matrix that solves Wahba's problem $ \min \|X_L - \popx_L\rot_L\|_F$, i.e. it is the procrustean transformation.
\end{definition}
Defining $Z$-misclustered, requires defining $c_i^R$ and  $\mu^z$ analogous to the previous definitions.
\begin{definition} 
The set of nodes $Z$-misclustered is
\begin{equation} \label{mzdef}
\mathscr{M}_z = \left\{i : \|c_i^R- z_i \mu^z\rot_R\|_2 > \|c_i^R - z_j \mu^z\rot_R\|_2 \ \mbox{ for any } \ z_j \ne z_i\right\},
\end{equation}
where $\rot_R$ is the orthonormal matrix that solves Wahba's problem $\min \|X_R - \popx_R\rot_R\|_F$, i.e. it is the procrustean transformation.
\end{definition}

\subsubsection{Asymptotic performance} \label{mainresult}

Define 
\[H = (Y^T\Theta_{Y,\tau} Y)^{1/2}B_L(Z^T\Theta_{Z,\tau}Z)^{1/2}.\] 
$H \in \R^{k_y \times k_z}$ shares same top $K$ singular values with the population graph Laplacian $\popl$. Define $H_{\cdot j}$ as the $j$th column of $H$, and define
\begin{equation} \label{maingammadef}
\gamma_z = \min_{i \ne j} \|H_{\cdot i} - H_{\cdot j}\|_2.
\end{equation}
When $k_z>k_y$, $\gamma_z$ controls the additional difficulty in estimating $Z$.  

Define $m_y$ as the minimum row length of $\popx_L$. Similarly define $m_z$ as the minimum row length of $\popx_R$. 
That is, 
\begin{equation} \label{leveragedef}
m_y = \min_{i = 1,..,N_r} ||[\popx_L]_i||_2, \quad\quad m_z = \min_{j = 1,..,N_c} ||[\popx_R]_j||_2.
\end{equation}
These are the minimum leverage scores for the matrices $\popl \popl^T$ and $\popl^T \popl$. 

The next theorem bounds the sizes of the sets of misclustered nodes,  $|\mathscr{M}_y|$ and $|\mathscr{M}_z|$.  

\begin{theorem}\label{clusterTheorem}
Suppose $A \in \R^{N_r \times N_c}$ is an adjacency matrix sampled from the Degree-Corrected Stochastic co-Blockmodel with $k_y$ left blocks and $k_x$ right blocks.   Let $K = \min \{k_y, k_z\} = k_y$.  Define $\popl$ as in Equation \ref{mainpopdef}. Define $\lambda_1 \ge \lambda_2 \ge \dots \ge \lambda_{K} >0$ as the $K$ nonzero singular values of $\popl$. 
Let $\mathscr{M}_y$ and $\mathscr{M}_z$  be the sets of $Y$- and $Z$-misclustered nodes (Equations \ref{mydef} and \ref{mzdef}) by DI-SIM.  
Let $\delta$ be the minimum expected row and column degree of $A$, that is $\delta = \min(\min_i \popo_{ii},  \min_j \popp_{jj})$.
Define $\gamma_z$, $m_y$ and $m_z$ as in Equations \ref{maingammadef} and \ref{leveragedef}.
For any $\epsilon > 0$, if $\delta + \tau >3\ln(N_r + N_c) + 3\ln (4/\epsilon)$, then with probability at least  $1- \epsilon$,
\begin{equation}\label{ybound}
 \frac{\mathscr{M}_y}{N_r} \le c_0(\alpha)\frac{K\ln(4(N_r + N_c)/\epsilon)}{N_r\lambda_K^2m_y^2(\delta + \tau)},
 \end{equation}
\begin{equation} \label{zbound}
 \frac{\mathscr{M}_z}{N_c} \le c_1(\alpha)\frac{K\ln(4(N_r + N_c)/\epsilon)}{N_c\lambda_K^2m_z^2\gamma_z^2(\delta + \tau)}.
\end{equation}
\end{theorem}

A proof of Theorem \ref{clusterTheorem} is contained in the appendix.

Because $\|\popx_L\|_F^2 = K$, the average leverage score $||[\popx_L]_i||_2$ is $\sqrt{K/N_r}$.  If the $m_y$ is of the same order, with $\lambda_K$ and $K$ fixed, then $ \frac{\mathscr{M}_y}{N_r}$ goes to zero when $\delta + \tau$ grows faster than $\ln(N_r + N_c)$.  In sparse graphs, $\delta$ is fixed and so $\tau$ must grow with $n$.  To ensure that $\lambda_K$ remains fixed while $\tau$ is growing, it is necessary for the \textit{average} degree to also grow.  

In many empirical networks, the vast majority of nodes have very small degrees; this is a regime in which $\delta$ is not growing.  In such networks, the bounds in Equations \eqref{ybound} and \eqref{zbound} are vacuous unless $\tau >0$.  
While these equations are upper bounds, the simulations in the appendix show that for sparse networks (i.e. $\delta$ small), these bounds align with the performance of \textsc{di-sim}.  Moreover, the performance of \textsc{di-sim} is drastically improves with statistical regularization.

These results highlight the sensitivity to the smallest leverage scores $m_y$ and $m_z$. When there are excessively small leverage scores, then the bound above can become meaningless.  However, a slight modification of \textsc{di-sim} that excludes the low leveraged points from the k-means step and the clustering results, obtains a vastly improved bound.  If one computes the leading singular vectors and only runs k-means on the  with the observations $i$ that satisfy $||[\popx_L]_i||_2 > \eta \sqrt{K/N}$, then the theoretical results are much improved.  Denote the nodes misclustered by this procedure as $\mathscr{M}_y^*$.  Let there be $N^*$ nodes with $||[\popx_L]_i||_2 > \eta \sqrt{K/N}$.  If $N/N^* = O(1)$  and the population eigengap $\lambda_K$ is not asymptotically diminishing, then 
\[ \frac{\mathscr{M}_y^*}{N^*} \le c_2(\alpha)\frac{\ln((N_r + N_c)/\epsilon)}{\eta^2(\delta + \tau)}.\]
The proof mimics the proof of Theorem \ref{clusterTheorem}.

In Theorem \ref{clusterTheorem}, the bound for  $\mathscr{M}_z$ exceeds the bound for $\mathscr{M}_y$ because the bound for $\mathscr{M}_z$ contains an additional term $\gamma_z$.  This asymmetry stems from allowing $k_z\ge k_y$.  In fact, if $k_y = k_z$, then $\gamma_z$ can be removed, making the bounds identical.
However, if $k_z > k_y$, then Rank$(\popl)$ is at most $k_y$.  So, the singular value decomposition  represents the data in $k_y$ dimensions and the k-means steps for both the left and the right clusters are done in $k_y$ dimensions.  In estimating $Y$, there is one dimension in the singular vector representation for each of the $k_y$ blocks.  At the same time, the singular value representation shoehorns the $k_z$ blocks in $Z$ into less than $k_z$ dimensions.  So, there is less space to separate each of the $k_z$ clusters, obscuring the estimation of $Z$.  

To further understand the bound in Theorem \ref{clusterTheorem},  define the following toy model. 

\begin{definition} \label{fourpara}
The \textbf{four parameter ScBM} is an ScBM parameterized by $K \in \mathbb{N},s \in \mathbb{N},r \in (0,1),$ and $p \in (0,1)$ such that $p+r \le 1$.  The matrices $Y,Z \in \{0,1\}^{n\times K}$ each contain $s$ ones in each column and $B = p I_K + r \textbf{1}_K \textbf{1}_K^T$.  
\end{definition}
In the four parameter ScBM, there are $K$ left- and right-blocks each with $s$ nodes and the node partitions in $Y$ and $Z$ are not necessarily related.  If $y_i = z_j$, then $P(i \rightarrow j) = p + r$.  Otherwise, $P(i \rightarrow j) = r$.

\begin{corollary} \label{fourcor}
Assume the four parameter ScBM, with same number of rows and columns, and $r$, $p$ fixed and $K$ growing with $N = Ks$.   Since $\delta$ is growing with $n$, set $\tau=0$. Then,
\[\lambda_K =   \frac{1}{K (r/p) + 1},\]
where $\lambda_K$ is the $K$th largest singular value of $\popl$. Moreover,
\[N^{-1}(|\mathscr{M}_y| + |\mathscr{M}_z|) = O_p\left(\frac{K^2 \log N}{N}\right).\]
The proportion of nodes that are misclustered converges to zero,  as long as number of clusters $K = o(\sqrt{ N/ \log N})$. 
\end{corollary}
The proof of Corollary \ref{fourcor} is contained in the supplementary materials, Section \ref{clusterAppendix}.

\section{Discussion} \label{conclusion}

\subsection{Related SVD methods} \label{sec:svd}

Several other researchers have used SVD to explore and understand different network features.  


\citet{kleinberg1999authoritative} proposed the concept of ``hubs and authorities" for hyperlink-induced topic search (HITS).  This algorithm that was a precursor to Google's PageRank algorithm (\citet{page1999pagerank}). The SVD plays a key role in this algorithm.
The SVD also played a key role in \citet{hoff2009multiplicative}, where the left and right singular vectors estimate ``sender-specific and receiver-specific latent nodal attributes".  Like \textsc{di-sim}, the algorithms in \citet{kleinberg1999authoritative} and \cite{hoff2009multiplicative} use the SVD to investigate asymmetric features of directed graphs.

\citet{dhillon2001co} suggested an algorithm similar to \textsc{di-sim} that was to be applied to bipartite graphs in which the rows and columns of $L$ correspond to different entities (e.g. documents and words).  There are three key differences between \textsc{di-sim} and the algorithm in \citet{dhillon2001co}.  First, \citet{dhillon2001co} does not use regularization.  So, the definition of $L$ remains the same, but $\tau =0$.  The regularization step helps \textsc{di-sim} when $L$ has highly localized singular vectors; this often happens when several nodes have very small degrees.  Second, \citet{dhillon2001co} does not project the rows of the singular vectors onto the sphere. The project step helps \textsc{di-sim} when the node degrees are highly heterogeneous.  Finally, to estimate $K$ clusters, \citet{dhillon2001co} only uses $\lceil \log_2 K \rceil$ singular vectors ($\lceil x \rceil$ is the smallest integer greater than $x$).   While it is much faster to only compute $\log_2 K$ singular vectors, there is additional information contained in the remaining top $K$ singular vectors.  For example, under the four parameter ScBM, $\lambda_2 = \dots = \lambda_K$.  As such, there is not an eigengap after the  $\lceil \log_2 K \rceil$th singular value.

SVD has been used in other forms of discrete data, most notably in correspondence analysis (CA).  In fact, \textsc{di-sim} normalizes the rows and columns in an identical fashion to CA.   CA has similarities to principal components analysis, but it is applicable to categorical data in contingency tables and is built on a beautiful set of algebraic ideas (\cite{holmes2006multivariate}). The methodology was first published in \citet{hirschfeld1935connection} and (like spectral clustering) it has been rediscovered and reapplied several times over (\cite{guttman1959metricizing}).   
While there exists a deep algorithmic, algebraic, and heuristic understanding of CA, it is rarely conceived through a statistical model; \citet{Goodman1986} is one exception.  \citet{wasserman1990correspondence} study how one could use CA to study relational data, but was particularly interested in two-way or bipartite networks. \citet{anderson1992building} mentions CA and visual inspection as one possible way to construct blocks in a Stochastic Blockmodel.  The previous CA literature has not explored the parameter estimation performance of CA under any of these models, nor has the literature explored the dual partitions under a directed graph.  Algorithmically, the CA literature does not employ the regularization step (using $\tau$) for sparse data.  Nor does it employ the projection step, where the rows of the singular vector matrices are normalized to have unit length.  This is a potentially fruitful area for further research in CA. 

In research that was  contemporaneous to this paper's tech report (\cite{rohe2012co}),  both \citet{choi} and \citet{flynn2012consistent}  studied likelihood formulations of co-clustering  in the network setting.  \citet{choi} studied a ``non-parametric" model that assumes the nodes are separately exchangeable.  This is a generalization of the Stochastic co-Blockmodel.  \citet{flynn2012consistent} uses a profile likelihood formulation to develop a consistent estimator of the Stochastic co-Blockmodel.

\subsection{Conclusion}

By extending both spectral clustering and the Stochastic Blockmodel to a co-clustering framework, this paper aims to better conceptualize clustering in directed graphs; co-clustering is a meaningful procedure for directed networks and  helps to guide the development of reasonable questions for network researchers.  

Given that empirical graphs can be sparse, with highly heterogeneous node degrees, we propose a novel spectral algorithm \textsc{di-sim} that incorporates both the regularization and projection steps.  
Section \ref{Examples} demonstrates how \textsc{di-sim}'s asymmetric analysis finds novel structure in three empirical networks.  In the Enron email network, it identifies Bill Williams, who was part of the conspiracy to manufacture energy shortages in Southern California.  In the political blog network, it identifies six asymmetric blogs.    Finally, in the \textit{c elegans} network \textsc{di-sim} identifies several neurons that form feedforward circuits.  In each of these examples, the conclusions are only feasible because the data analysis leverages the edge asymmetries.  

Investigating the statistical properties of \textsc{di-sim} required several theoretical novelties that build on the extensive literature for spectral algorithms. The results highlight the importance of regularization and the statistical leverage scores. Importantly, because of the regularization, the convergence of the singular vectors does not require a growing minimum degree.  Moreover, because the theory accommodates a ``degree corrected" model, it was necessary to project the rows of $X_L$ and $X_R$ onto the sphere.  Finally, these results extend to bipartite graphs, where the rows and columns of the adjacency matrix index different sets of objects.

\textbf{Acknowledgements:} Thank you David Gleich for your thoughtful questions and helpful references.  Thank you Sara Fernandes-Taylor and Zoe Russek for your helpful comments. Thank you Susan Holmes for the helpful references.  While Karl Rohe was a graduate student, he was partially supported by an NSF VIGRE Graduate Fellowship at UC Berkeley and ARO grant W911NF-11-1-0114.  More recently, NSF DMS-1309998 has supported this research. Tai Qin is supported by DMS-1308877.  Bin Yu is partially supported by NSF grants SES-0835531 (CDI), DMS-1107000, 0939370 CCF, and ARO grant W911NF-11-1-0114.

\bibliographystyle{natbib}
\bibliography{references}

\begin{thebibliography}{}

\bibitem[Adamic and Glance(2005)Adamic and Glance]{Adamic}
Adamic, L.~A. and Glance, N. (2005).
\newblock The political blogosphere and the 2004 us election: divided they
  blog.
\newblock In {\em Proceedings of the 3rd international workshop on Link
  discovery\/}, pages 36--43. ACM.

\bibitem[Ames and Vavasis(2010)Ames and Vavasis]{ames2010convex}
Ames, B.~P. and Vavasis, S.~A. (2010).
\newblock Convex optimization for the planted k-disjoint-clique problem.
\newblock {\em arXiv preprint arXiv:1008.2814\/}.

\bibitem[Amini {\em et~al.}(2013)Amini, Chen, Bickel, Levina, {\em
  et~al.}]{amini2013pseudo}
Amini, A.~A., Chen, A., Bickel, P.~J., Levina, E., {\em et~al.} (2013).
\newblock Pseudo-likelihood methods for community detection in large sparse
  networks.
\newblock {\em The Annals of Statistics\/}, {\bf 41}(4), 2097--2122.

\bibitem[Andersen {\em et~al.}(2006)Andersen, Chung, and
  Lang]{andersen2006local}
Andersen, R., Chung, F., and Lang, K. (2006).
\newblock Local graph partitioning using pagerank vectors.
\newblock In {\em Foundations of Computer Science, 2006. FOCS'06. 47th Annual
  IEEE Symposium on\/}, pages 475--486. IEEE.

\bibitem[Anderson {\em et~al.}(1992)Anderson, Wasserman, and
  Faust]{anderson1992building}
Anderson, C.~J., Wasserman, S., and Faust, K. (1992).
\newblock Building stochastic blockmodels.
\newblock {\em Social networks\/}, {\bf 14}(1), 137--161.

\bibitem[Banerjee {\em et~al.}(2004)Banerjee, Dhillon, Ghosh, Merugu, and
  Modha]{banerjee2004generalized}
Banerjee, A., Dhillon, I., Ghosh, J., Merugu, S., and Modha, D. (2004).
\newblock A generalized maximum entropy approach to bregman co-clustering and
  matrix approximation.
\newblock In {\em Proceedings of the tenth ACM SIGKDD international conference
  on Knowledge discovery and data mining\/}, pages 509--514. ACM.

\bibitem[Bhattacharyya and Bickel(2014)Bhattacharyya and Bickel]{sharmo}
Bhattacharyya, S. and Bickel, P.~J. (2014).
\newblock Community detection in networks using graph distance.
\newblock {\em arXiv preprint arXiv:1401.3915\/}.

\bibitem[Bisson and Hussain(2008)Bisson and Hussain]{bisson2008chi}
Bisson, G. and Hussain, F. (2008).
\newblock Chi-sim: A new similarity measure for the co-clustering task.
\newblock In {\em Machine Learning and Applications, 2008. ICMLA'08. Seventh
  International Conference on\/}, pages 211--217. IEEE.

\bibitem[Borchers(2012)Borchers]{kmpp}
Borchers, H.~W. (2012).
\newblock [r] k-means++.
  \texttt{https://stat.ethz.ch/pipermail/r-help/2012-January/300051.html}.

\bibitem[Chaudhuri {\em et~al.}(2012)Chaudhuri, Graham, and
  Tsiatas]{chaudhuri2012spectral}
Chaudhuri, K., Graham, F.~C., and Tsiatas, A. (2012).
\newblock Spectral clustering of graphs with general degrees in the extended
  planted partition model.
\newblock {\em Journal of Machine Learning Research-Proceedings Track\/}, {\bf
  23}, 35--1.

\bibitem[Chen {\em et~al.}(2012)Chen, Amini, Bickel, and
  Levina]{chen2012fitting}
Chen, A., Amini, A.~A., Bickel, P.~J., and Levina, E. (2012).
\newblock Fitting community models to large sparse networks.
\newblock {\em arXiv preprint arXiv:1207.2340\/}.

\bibitem[Choi {\em et~al.}(2012)Choi, Wolfe, and Airoldi]{choi2010stochastic}
Choi, D., Wolfe, P., and Airoldi, E. (2012).
\newblock {Stochastic blockmodels with growing number of classes}.
\newblock {\em Biometrica (in press)\/}.

\bibitem[Chung(1997)Chung]{chung1997spectral}
Chung, F. (1997).
\newblock {\em {Spectral graph theory}\/}.
\newblock Amer Mathematical Society.

\bibitem[Chung and Radcliffe(2011)Chung and Radcliffe]{chung2011spectra}
Chung, F. and Radcliffe, M. (2011).
\newblock On the spectra of general random graphs.
\newblock {\em the electronic journal of combinatorics\/}, {\bf 18}(P215), 1.

\bibitem[Chung and Lu(2006)Chung and Lu]{chung2006complex}
Chung, F. R.~K. and Lu, L. (2006).
\newblock {\em Complex graphs and networks\/}.
\newblock Number 107. American Mathematical Soc.

\bibitem[Cohen(2009)Cohen]{cohen}
Cohen, W.~W. (2009).
\newblock Enron email dataset.

\bibitem[Coja-Oghlan and Lanka(2009)Coja-Oghlan and Lanka]{coja2009finding}
Coja-Oghlan, A. and Lanka, A. (2009).
\newblock Finding planted partitions in random graphs with general degree
  distributions.
\newblock {\em SIAM Journal on Discrete Mathematics\/}, {\bf 23}(4),
  1682--1714.

\bibitem[Dasgupta {\em et~al.}(2004)Dasgupta, Hopcroft, and
  McSherry]{dasgupta2004spectral}
Dasgupta, A., Hopcroft, J.~E., and McSherry, F. (2004).
\newblock Spectral analysis of random graphs with skewed degree distributions.
\newblock In {\em Foundations of Computer Science, 2004. Proceedings. 45th
  Annual IEEE Symposium on\/}, pages 602--610. IEEE.

\bibitem[Dhillon(2001)Dhillon]{dhillon2001co}
Dhillon, I. (2001).
\newblock Co-clustering documents and words using bipartite spectral graph
  partitioning.
\newblock In {\em Proceedings of the seventh ACM SIGKDD international
  conference on Knowledge discovery and data mining\/}, pages 269--274. ACM.

\bibitem[Diaconis and Janson(2007)Diaconis and Janson]{persi}
Diaconis, P. and Janson, S. (2007).
\newblock Graph limits and exchangeable random graphs.
\newblock {\em arXiv preprint arXiv:0712.2749\/}.

\bibitem[Egan(2005)Egan]{nytimes}
Egan, T. (2005).
\newblock Tapes reveal enron took a role in crisis.
\newblock [Online; accessed 16-July-2013].

\bibitem[Fiedler(1973)Fiedler]{fiedler1973algebraic}
Fiedler, M. (1973).
\newblock {Algebraic connectivity of graphs}.
\newblock {\em Czechoslovak Mathematical Journal\/}, {\bf 23}(2), 298--305.

\bibitem[Flynn and Perry(2012)Flynn and Perry]{flynn2012consistent}
Flynn, C.~J. and Perry, P.~O. (2012).
\newblock Consistent biclustering.
\newblock {\em arXiv preprint arXiv:1206.6927\/}.

\bibitem[Freitag(2004)Freitag]{freitag2004trained}
Freitag, D. (2004).
\newblock Trained named entity recognition using distributional clusters.
\newblock In {\em Proceedings of EMNLP\/}, volume~4, pages 262--269.

\bibitem[Goodman(1986)Goodman]{Goodman1986}
Goodman, L.~A. (1986).
\newblock Some useful extensions of the usual correspondence analysis approach
  and the usual log-linear models approach in the analysis of contingency
  tables.
\newblock {\em International Statistical Review / Revue Internationale de
  Statistique\/}, {\bf 54}(3), pp. 243--270.

\bibitem[Guttman(1959)Guttman]{guttman1959metricizing}
Guttman, L. (1959).
\newblock {Metricizing rank-ordered or unordered data for a linear factor
  analysis}.
\newblock {\em Sankhy{\=a}: The Indian Journal of Statistics (1933-1960)\/},
  {\bf 21}(3/4), 257--268.

\bibitem[Hartigan(1972)Hartigan]{hartigan1972direct}
Hartigan, J. (1972).
\newblock Direct clustering of a data matrix.
\newblock {\em Journal of the American Statistical Association\/}, pages
  123--129.

\bibitem[Hirschfeld(1935)Hirschfeld]{hirschfeld1935connection}
Hirschfeld, H. (1935).
\newblock {A connection between correlation and contingency}.
\newblock {\em Mathematical Proceedings of the Cambridge Philosophical
  Society\/}, {\bf 31}(04), 520--524.

\bibitem[Hoff(2009)Hoff]{hoff2009multiplicative}
Hoff, P. (2009).
\newblock {Multiplicative latent factor models for description and prediction
  of social networks}.
\newblock {\em Computational \& Mathematical Organization Theory\/}, {\bf
  15}(4), 261--272.

\bibitem[Hoff {\em et~al.}(2002)Hoff, Raftery, and Handcock]{hoff2002latent}
Hoff, P., Raftery, A., and Handcock, M. (2002).
\newblock {Latent space approaches to social network analysis}.
\newblock {\em Journal of the American Statistical Association\/}, {\bf
  97}(460), 1090--1098.

\bibitem[Holland {\em et~al.}(1983)Holland, Laskey, and
  Leinhardt]{holland1983stochastic}
Holland, P., Laskey, K., and Leinhardt, S. (1983).
\newblock {Stochastic blockmodels: Some first steps}.
\newblock {\em Social Networks\/}, {\bf 5}, 109--137.

\bibitem[Holmes(2006)Holmes]{holmes2006multivariate}
Holmes, S. (2006).
\newblock {Multivariate analysis: The french way}.
\newblock {\em Festschrift for David Freedman\/}.

\bibitem[Horn and Johnson(2005)Horn and Johnson]{horn2005matrix}
Horn, R. and Johnson, C. (2005).
\newblock {\em {Matrix analysis}\/}.
\newblock Cambridge university press.

\bibitem[Jarrell {\em et~al.}(2012)Jarrell, Wang, Bloniarz, Brittin, Xu,
  Thomson, Albertson, Hall, and Emmons]{jarrell2012connectome}
Jarrell, T.~A., Wang, Y., Bloniarz, A.~E., Brittin, C.~A., Xu, M., Thomson,
  J.~N., Albertson, D.~G., Hall, D.~H., and Emmons, S.~W. (2012).
\newblock The connectome of a decision-making neural network.
\newblock {\em Science\/}, {\bf 337}(6093), 437--444.

\bibitem[Jin(2015)Jin]{jinScore}
Jin, J. (2015).
\newblock Fast community detection by score.
\newblock {\em The Annals of Statistics\/}, {\bf 43}(1), 57--89.

\bibitem[Joseph and Yu(2014)Joseph and Yu]{joseph2013impact}
Joseph, A. and Yu, B. (2014).
\newblock Impact of regularization on spectral clustering.
\newblock {\em arXiv preprint arXiv:1312.1733\/}.

\bibitem[Karrer and Newman(2011)Karrer and Newman]{karrer2011stochastic}
Karrer, B. and Newman, M.~E. (2011).
\newblock Stochastic blockmodels and community structure in networks.
\newblock {\em Physical Review E\/}, {\bf 83}(1), 016107.

\bibitem[Kleinberg(1999)Kleinberg]{kleinberg1999authoritative}
Kleinberg, J. (1999).
\newblock Authoritative sources in a hyperlinked environment.
\newblock {\em Journal of the ACM (JACM)\/}, {\bf 46}(5), 604--632.

\bibitem[Koltchinskii and Gin{\'e}(2000)Koltchinskii and
  Gin{\'e}]{Koltchinskii2000}
Koltchinskii, V. and Gin{\'e}, E. (2000).
\newblock Random matrix approximation of spectra of integral operators.
\newblock {\em Bernoulli\/}, {\bf 6}(1), 113--167.

\bibitem[Krzakala {\em et~al.}(2013)Krzakala, Moore, Mossel, Neeman, Sly,
  Zdeborov{\'a}, and Zhang]{krzakala2013spectral}
Krzakala, F., Moore, C., Mossel, E., Neeman, J., Sly, A., Zdeborov{\'a}, L.,
  and Zhang, P. (2013).
\newblock Spectral redemption in clustering sparse networks.
\newblock {\em Proceedings of the National Academy of Sciences\/}, {\bf
  110}(52), 20935--20940.

\bibitem[Kumar {\em et~al.}(2004)Kumar, Sabharwal, and Sen]{kumar2004simple}
Kumar, A., Sabharwal, Y., and Sen, S. (2004).
\newblock A simple linear time (1+ $\varepsilon$)-approximation algorithm for
  geometric k-means clustering in any dimensions.
\newblock In {\em Proceedings-Annual Symposium on Foundations of Computer
  Science\/}, pages 454--462. IEEE.

\bibitem[Lei and Rinaldo(2013)Lei and Rinaldo]{lei2013consistency}
Lei, J. and Rinaldo, A. (2013).
\newblock Consistency of spectral clustering in sparse stochastic block models.
\newblock {\em arXiv preprint arXiv:1312.2050\/}.

\bibitem[Lei and Rinaldo(2015)Lei and Rinaldo]{lei2014consistency}
Lei, J. and Rinaldo, A. (2015).
\newblock Consistency of spectral clustering in stochastic block models.
\newblock {\em The Annals of Statistics\/}, {\bf 43}(1), 215--237.

\bibitem[Leicht and Newman(2008)Leicht and Newman]{leicht2008community}
Leicht, E. and Newman, M. (2008).
\newblock Community structure in directed networks.
\newblock {\em Physical Review Letters\/}, {\bf 100}(11), 118703.

\bibitem[Leskovec {\em et~al.}(2008)Leskovec, Lang, Dasgupta, and
  Mahoney]{leskovec2008statistical}
Leskovec, J., Lang, K., Dasgupta, A., and Mahoney, M. (2008).
\newblock {Statistical properties of community structure in large social and
  information networks}.
\newblock In {\em Proceeding of the 17th international conference on World Wide
  Web\/}, pages 695--704. ACM.

\bibitem[Madeira and Oliveira(2004)Madeira and
  Oliveira]{madeira2004biclustering}
Madeira, S. and Oliveira, A. (2004).
\newblock Biclustering algorithms for biological data analysis: a survey.
\newblock {\em Computational Biology and Bioinformatics, IEEE/ACM Transactions
  on\/}, {\bf 1}(1), 24--45.

\bibitem[Madeira {\em et~al.}(2010)Madeira, Teixeira, Sa-Correia, and
  Oliveira]{madeira2010identification}
Madeira, S., Teixeira, M., Sa-Correia, I., and Oliveira, A. (2010).
\newblock Identification of regulatory modules in time series gene expression
  data using a linear time biclustering algorithm.
\newblock {\em Computational Biology and Bioinformatics, IEEE/ACM Transactions
  on\/}, {\bf 7}(1), 153--165.

\bibitem[Mahoney(2011)Mahoney]{mahoney2011randomized}
Mahoney, M.~W. (2011).
\newblock Randomized algorithms for matrices and data.
\newblock {\em Foundations and Trends{\textregistered} in Machine Learning\/},
  {\bf 3}(2), 123--224.

\bibitem[McSherry(2001)McSherry]{mcsherry2001spectral}
McSherry, F. (2001).
\newblock Spectral partitioning of random graphs.
\newblock In {\em Foundations of Computer Science, 2001. Proceedings. 42nd IEEE
  Symposium on\/}, pages 529--537. IEEE.

\bibitem[Page {\em et~al.}(1999)Page, Brin, Motwani, and
  Winograd]{page1999pagerank}
Page, L., Brin, S., Motwani, R., and Winograd, T. (1999).
\newblock The pagerank citation ranking: Bringing order to the web.

\bibitem[Perry and Wolfe(2013)Perry and Wolfe]{perry2010point}
Perry, P.~O. and Wolfe, P.~J. (2013).
\newblock Point process modelling for directed interaction networks.
\newblock {\em Journal of the Royal Statistical Society: Series B (Statistical
  Methodology)\/}.

\bibitem[Qin and Rohe(2013)Qin and Rohe]{qin}
Qin, T. and Rohe, K. (2013).
\newblock Regularized spectral clustering under the degree-corrected stochastic
  blockmodel.
\newblock {\em Advances in Neural Information Processing Systems\/}.

\bibitem[Rohe and Yu(2012)Rohe and Yu]{rohe2012co}
Rohe, K. and Yu, B. (2012).
\newblock Co-clustering for directed graphs; the stochastic co-blockmodel and a
  spectral algorithm.
\newblock {\em arXiv preprint arXiv:1204.2296\/}.

\bibitem[Rohe {\em et~al.}(2011)Rohe, Chatterjee, and Yu]{rohe}
Rohe, K., Chatterjee, S., and Yu, B. (2011).
\newblock Spectral clustering and the high-dimensional stochastic blockmodel.
\newblock {\em The Annals of Statistics\/}, {\bf 39}(4), 1878--1915.

\bibitem[Rohe {\em et~al.}(2012)Rohe, Qin, and Fan]{hdsbm}
Rohe, K., Qin, T., and Fan, H. (2012).
\newblock The highest dimensional stochastic blockmodel with a regularized
  estimator.
\newblock {\em arXiv preprint arXiv:1206.2380\/}.

\bibitem[Rohwer and Freitag(2004)Rohwer and Freitag]{rohwer2004towards}
Rohwer, R. and Freitag, D. (2004).
\newblock Towards full automation of lexicon construction.
\newblock In {\em Proceedings of the HLT-NAACL Workshop on Computational
  Lexical Semantics\/}, pages 9--16. Association for Computational Linguistics.

\bibitem[Sarkar and Bickel(2013)Sarkar and Bickel]{sarkar2013role}
Sarkar, P. and Bickel, P.~J. (2013).
\newblock Role of normalization in spectral clustering for stochastic
  blockmodels.
\newblock {\em arXiv preprint arXiv:1310.1495\/}.

\bibitem[Steinhaus(1956)Steinhaus]{steinhaus1956division}
Steinhaus, H. (1956).
\newblock {Sur la division des corp materiels en parties}.
\newblock {\em Bull. Acad. Polon. Sci\/}, {\bf 1}, 801--804.

\bibitem[Sussman {\em et~al.}(2012)Sussman, Tang, Fishkind, and
  Priebe]{sussman2011consistent}
Sussman, D.~L., Tang, M., Fishkind, D.~E., and Priebe, C.~E. (2012).
\newblock A consistent adjacency spectral embedding for stochastic blockmodel
  graphs.
\newblock {\em Journal of the American Statistical Association\/}, {\bf
  107}(499), 1119--1128.

\bibitem[Tanay {\em et~al.}(2004)Tanay, Sharan, Kupiec, and
  Shamir]{tanay2004revealing}
Tanay, A., Sharan, R., Kupiec, M., and Shamir, R. (2004).
\newblock Revealing modularity and organization in the yeast molecular network
  by integrated analysis of highly heterogeneous genomewide data.
\newblock {\em Proceedings of the National Academy of Sciences of the United
  States of America\/}, {\bf 101}(9), 2981.

\bibitem[Tanay {\em et~al.}(2005)Tanay, Sharan, and
  Shamir]{tanay2005biclustering}
Tanay, A., Sharan, R., and Shamir, R. (2005).
\newblock Biclustering algorithms: A survey.
\newblock {\em Handbook of computational molecular biology\/}, {\bf 9}, 26--1.

\bibitem[von Luxburg(2007)von Luxburg]{vonluxburg2007tsc}
von Luxburg, U. (2007).
\newblock {A tutorial on spectral clustering}.
\newblock {\em Statistics and Computing\/}, {\bf 17}(4), 395--416.

\bibitem[Wang and Wong(1987)Wang and Wong]{wang1987stochastic}
Wang, Y. and Wong, G. (1987).
\newblock {Stochastic blockmodels for directed graphs}.
\newblock {\em Journal of the American Statistical Association\/}, {\bf
  82}(397), 8--19.

\bibitem[Wasserman {\em et~al.}(1990)Wasserman, Faust, and
  Galaskiewicz]{wasserman1990correspondence}
Wasserman, S., Faust, K., and Galaskiewicz, J. (1990).
\newblock Correspondence and canonical analysis of relational data.
\newblock {\em Journal of Mathematical Sociology\/}, {\bf 15}(1), 11--64.

\bibitem[White {\em et~al.}(1976)White, Boorman, and Breiger]{white1976social}
White, H., Boorman, S., and Breiger, R. (1976).
\newblock {Social structure from multiple networks. I. Blockmodels of roles and
  positions}.
\newblock {\em American Journal of Sociology\/}, {\bf 81}(4), 730--780.

\bibitem[White {\em et~al.}(1986)White, Southgate, Thomson, and
  Brenner]{white1986structure}
White, J.~G., Southgate, E., Thomson, J.~N., and Brenner, S. (1986).
\newblock The structure of the nervous system of the nematode caenorhabditis
  elegans.
\newblock {\em Philosophical Transactions of the Royal Society of London. B,
  Biological Sciences\/}, {\bf 314}(1165), 1--340.

\bibitem[Wikipedia(2013)Wikipedia]{wikipedia}
Wikipedia (2013).
\newblock Enron --- {W}ikipedia{,} the free encyclopedia.
\newblock [Online; accessed 16-July-2013].

\bibitem[Wolfe and Choi(2014)Wolfe and Choi]{choi}
Wolfe, P. and Choi, D. (2014).
\newblock Co-clustering separately exchangeable network data.
\newblock {\em Annals of Statistics\/}.

\bibitem[Zhao {\em et~al.}(2012)Zhao, Levina, and Zhu]{zhao2012consistency}
Zhao, Y., Levina, E., and Zhu, J. (2012).
\newblock Consistency of community detection in networks under degree-corrected
  stochastic block models.
\newblock {\em The Annals of Statistics\/}, {\bf 40}(4), 2266--2292.

\bibitem[Zhou {\em et~al.}(2007)Zhou, Goldberg, Magdon-Ismail, and
  Wallace]{zhou2007strategies}
Zhou, Y., Goldberg, M., Magdon-Ismail, M., and Wallace, W. (2007).
\newblock {Strategies for cleaning organizational emails with an application to
  enron email dataset}.
\newblock In {\em 5th Conf. of North American Association for Computational
  Social and Organizational Science\/}. Citeseer.

\end{thebibliography}

\newpage
\appendix

\begin{center}
\begin{LARGE}
\textbf{Supplementary materials}
\end{LARGE}
\end{center}

\section{Simulation}\label{sims}

The theoretical results of Theorem \ref{clusterTheorem} identify (1) the expected node degree and (2) the spectral gap as essential parameters that control the clustering performance of \textsc{di-sim}.
The simulations investigate \textsc{di-sim}'s non-asymptotic sensitivity to these quantities under the four parameter Stochastic Co-Blockmodel (Definition \ref{fourpara}).   Moreover, the simulations investigate the performance under the model without degree correction and with degree correction.   

Both simulations use $k = 5$ blocks for both $Y$ and $Z$.  Each of the five blocks contains $400$ nodes.  So, $n = 2000$.  When the model is degree corrected, $\theta_1, \dots, \theta_n$ are iid with $\theta_i \stackrel{ d}{=} \sqrt{Z + .169}$ where $Z \sim$ exponential(1). The addition of $.169$ ensures that $\E(\theta_i) \approx 1$ and thus the expected degrees are unchanged between the degree corrected model and the model without degree correction.

In the first simulation, the expected node degree is represented on the horizontal axis; the out of block probability $r$ and the in block probability $p+r$ change in a way that keeps the spectral gap of $\popl$ fixed across the horizontal axis. 
In the second simulation, the spectral gap is represented on the horizontal axis; the probabilities $p$ and $r$ change so that  the expected degree $pk + rn$ remains fixed at twenty. In both simulations, the partition matrices $Y$ and $Z$ are sampled independently and uniformly over the set of matrices with $s=400$ and $k=5$. 

To design the parameter settings of $p$ and $r$, 
note that the population graph Laplacian  $\popl$ is a rank $k$ matrix.  So, its $k+1$ eigenvalue is $\lambda_{k+1} = 0$ and the spectral gap is $\lambda_k - \lambda_{k+1}  = \lambda_k$.
Corollary \ref{fourcor} says that the $k$th eigenvalue of $\popl$ for $\tau =0$ is 
\[\lambda_k =   \frac{1}{k (r/p) + 1}.\]
To keep the spectral gap $\lambda_k$ fixed, it is equivalent to keeping $r/p$ fixed.

We use the $k$-means++ algorithm (\citet{kumar2004simple}, \citet{kmpp}) with ten initializations.  Only the results for $Y$-misclustered (Definition \ref{mydef}) are reported.  Code is provided at \texttt{http://www.stat.wisc.edu/$\sim$karlrohe/}.

\subsubsection{Simulation 1}
This simulation investigates the sensitivity of \textsc{di-sim}  to a diminishing number of edges.  
Figure \ref{degfig} displays the simulation results for a sequence of nine equally spaced values of the expected degree between $5$ and $16$.   To decrease the variability of the plot, each simulation was run twenty times;  only the average is displayed. The solid line corresponds to setting the regularization parameter equal to zero ($\tau = 0$).  The line with longer dashes represents $\tau = 1$.  The line with small dashes represents the average degree, $\tau = \tfrac{1}{n} \sum_i P_{ii} $.

\begin{figure}[h!]
\begin{center}
\scalebox{.99}{
\includegraphics[width = 6in]{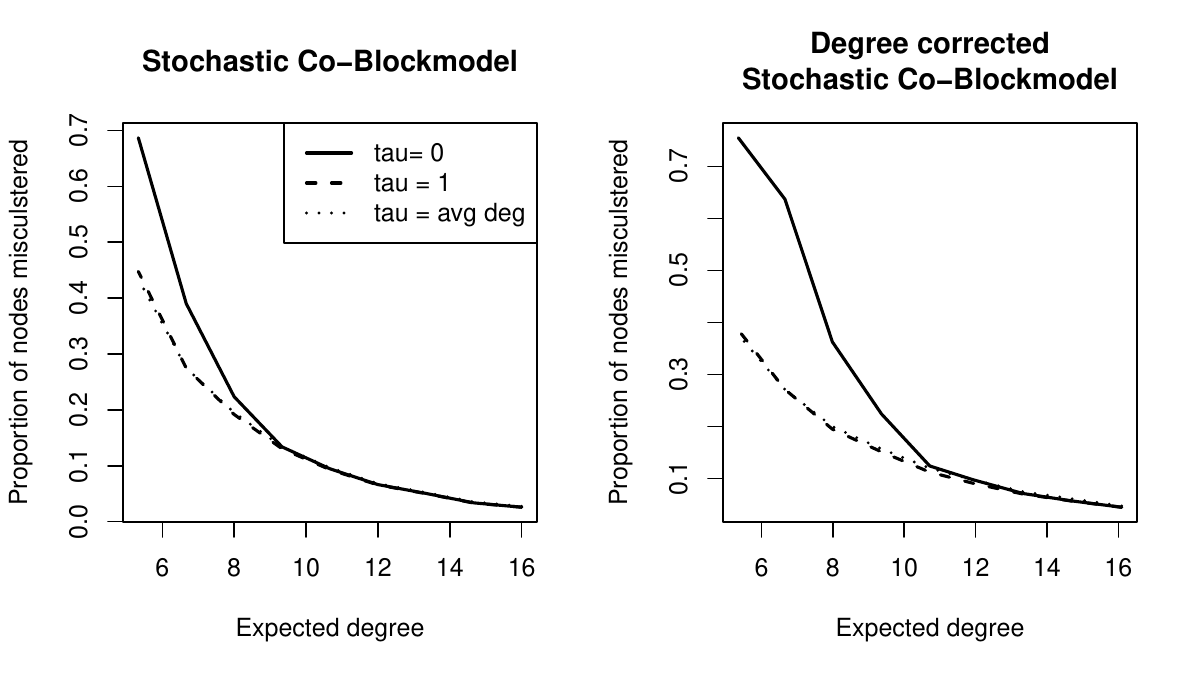}
}
 \vspace{-5 mm}
 \caption{
In the simulation on the left, the data comes from the four parameter Stochastic Co-Blockmodel.  On the right, the data comes from the same model, but with degree correction.  The $\theta_i$ parameters have expectation one.  In both models, $k=5$ and $s=400$.  The probabilities $p$ and $r$ vary such that $p = 5r$, keeping the spectral gap fixed at  $\lambda_k = 1/2$.  This simulation shows that for small expected degree, regularization decreases the proportion of nodes that are misclustered. Moreover, the benefits of regularization are more pronounced under the degree corrected model.
 } \label{degfig}
\end{center}
\end{figure}

Figure \ref{degfig} demonstrates two things.  First, the number of misclustered nodes increases as the expected degree goes to zero.  Second, regularization decreases the number of misclustered nodes for small values of the expected degree.

\subsubsection{Simulation 2}

This simulation investigates the sensitivity of \textsc{di-sim}  to a diminishing spectral gap $\lambda_k$.  
Figure \ref{degfig} displays the simulation results for a sequence of nine equally spaced values of the spectral gap, between $.3$ and $.6$.  In each simulation, the expected degree is held constant at twenty.  To decrease the variability, each simulation was run twenty times;  only the average is displayed. The solid line corresponds to setting the regularization parameter equal to zero ($\tau = 0$).  The line with longer dashes represents $\tau = 1$.  The line with small dashes represents the average degree, $\tau = \tfrac{1}{n} \sum_i P_{ii} $.  

\begin{figure}[h!]
\begin{center}
\scalebox{.99}{
\includegraphics[width = 6in]{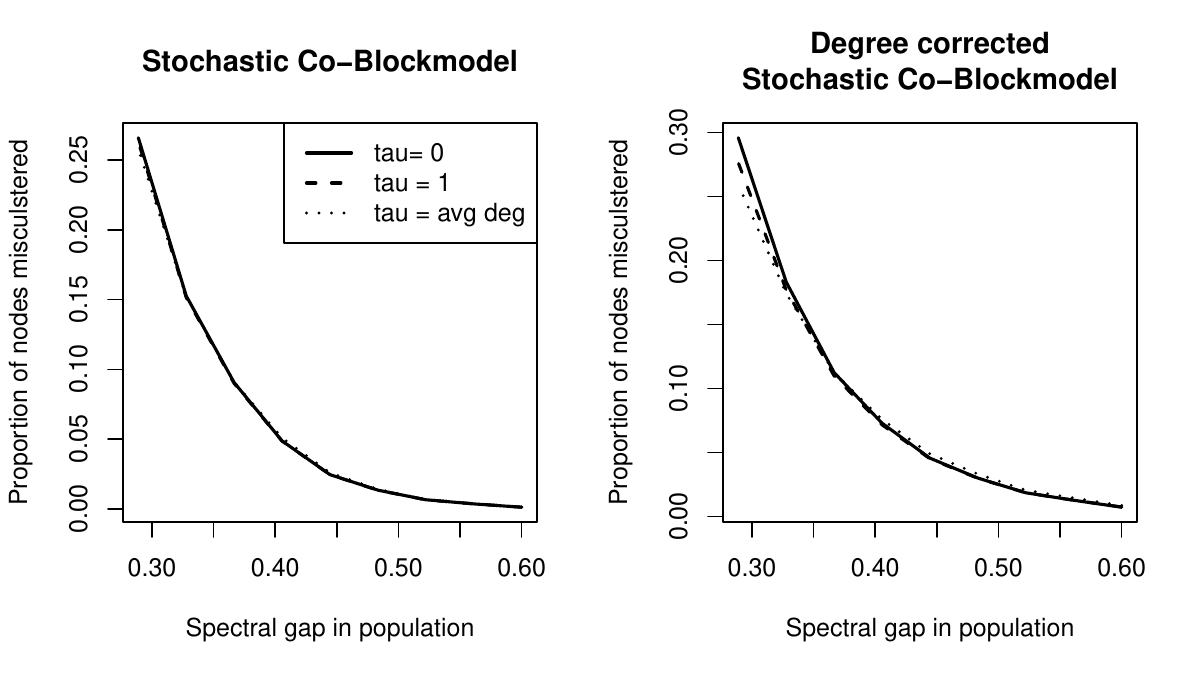}
}
 \vspace{-5 mm}
 \caption{
In the simulation on the left, the data comes from the four parameter Stochastic Co-Blockmodel.  On the right, the data comes from the same model, but with degree correction.  The $\theta_i$ parameters have expectation one.  In both models, $k=5$ and $s=400$.  
The spectral gap, displayed on the horizontal axis, changes because the probabilities $p$ and $r$ change.  The values of $p$ and $r$ vary in a way that keeps the expected degree fixed at twenty for all simulations.  Without degree correction, the three separate lines are difficult to distinguish because they are nearly identical.  Under the degree corrected model, regularization improves performance when the spectral gap is small. 
%
%
 } \label{sigfig}
\end{center}
\end{figure}

Figure \ref{degfig} demonstrates two things.  First, the number of misclustered nodes increases as the spectral gap goes to zero.  Second, regularization yields slight benefits when the spectral gap is small and the model is degree corrected.

\section{Directed latent space model}

The following definition of the directed latent space model is motivated by the Aldous-Hoover representation for infinite exchangeable arrays and the latent space model proposed by \citet{hoff2002latent}.  It specifies the distribution of the random directed adjacency matrix $A \in \{0,1\}^{n \times n}$. 

\begin{definition} 
The random adjacency matrix $A$ is from the \textbf{directed latent space model} if and only if 
\[\bP(A| \{z_i, y_i\}_{i=1}^n) = \prod_{i < j}\bP(A_{ij} |y_i, z_j)\]
where $\{z_i, y_i\}_{i=1}^n \subset \R^k \times \R^k$ are pairs of random vectors that are independent across $i =1, \dots, n$. 
\end{definition}
In this definition, $\bP(A_{ij} |y_i, z_j)$ is the probability mass function of $A_{ij}$ conditioned on $y_i$ and $z_j$.  Define $Y \in \R^{n \times k}$ such that its $i$th row is  $y_i$ for all $i\in V$.  Similarly, define $Z \in \R^{n\times k}$ such that its $i$th row is $z_i$. \textbf{Throughout this paper we condition on $Y$ and $Z$}. Because $\bP(A_{ij} = 1|Y,Z) = \E(A_{ij}|Y,Z)$, the model is then completely parametrized by the matrix 
\[\popa = \E(A|Y,Z) \in \R^{n \times n},\]
where $\popa$ depends on $Y$ and $Z$, but this is dropped for notational convenience. 

The Stochastic Blockmodel,  introduced by \citet{holland1983stochastic}, is a specific latent space model with well defined communities.  The following definition extends the Stochastic Blockmodel to allow for the asymmetric communities discussed in the previous section. 

\begin{definition}
The \textbf{Stochastic co-Blockmodel} with $k$ blocks is a directed latent space model with 
\[    \popa = Y B Z^T, \]
where $Y, Z \in \{0,1\}^{n \times k}$ both have exactly one 1 in each row and at least one 1 in each column and $B \in [0,1]^{k \times k}$ is full rank. 
\end{definition}

\section{Convergence of Singular Vectors}

The classical spectral clustering algorithm above can be divided into two steps: (1) find the eigendecomposition of $L$ and (2) run $k$-means.  Several previous papers have studied the estimation performance of the classical spectral clustering algorithm under a standard social network model. However, due to the asymmetry of $A$, previous proof techniques can not be directly applied to study the singular vectors for \textsc{di-sim}. In this analysis, we (a) symmetrize the graph Laplacian, (b) apply modern matrix concentration techniques to this symmetrized version of the graph Laplacian, and (c) apply an updated version of the Davis-Kahn theorem to bound the distance between the singular spaces of the empirical and population Laplacian.

For simplicity, from now on let $L$ denote the regularized graph Laplacian.

Define the symmetrized version of $L$ and $\popl$ as 
\[ \tilde{L} = \left( \begin{array}{cc}
   0 & L \\
   L^T & 0 \end{array} \right), \quad\quad \tilde{\popl} = \left( \begin{array}{ccc}
   0 & \popl \\
   \popl^T & 0 \end{array} \right).\] 
The next theorem gives a sharp bound between $\tilde L$ and $\tilde\popl$.

\begin{theorem} \label{Concentration}
(Concentration of $L$) Let $G$ be a random graph, with independent edges and $pr(v_i \sim v_j) = p_{ij}$. Let $\delta$ be the minimum expected row and column degree of $G$, that is $\delta = \min(\min_i \popo_{ii},  \min_j \popp_{jj})$.
For any $\epsilon > 0$, if $\delta + \tau >3\ln(N_r+N_c) + 3\ln (4/\epsilon)$, then with probability at least  $1- \epsilon$,
 \begin{equation}
 \|\tilde L - \tilde\popl\| \leq 4 \sqrt{\frac{3\ln(4(N_r+N_c)/\epsilon)}{\delta+\tau}}.
\end{equation}
\end{theorem}

 \begin{proof} Let $C =  \popp_{\tau}^{-\frac{1}{2}}A \popo_{\tau}^{-\frac{1}{2}}$ and define $\tilde C$ in the same way as $\tilde L$.
 Then $\|\tilde L - \tilde\popl\| \le \|\tilde C - \tilde\popl\| + \|\tilde L - \tilde C\|$. We bound the two terms separately. 
 
For the first term, we apply the following concentration inequality for matrices, see for example \citet{chung2011spectra}.
  \begin{lemma} \label{lem:61}
Let $X_1, X_2,..., X_m$ be independent random $N \times N$ Hermitian matrices. Moreover, assume that $\|X_i - \E(X_i)\| \le M$ for all $i$, and $v^2  = \|\sum var(X_i)\|$. Let $X = \sum X_i$. Then for any $a > 0$,
\[pr(\|X - \E(X)\| \ge a) \le 2N \exp \bigg(-\frac{a^2}{2v^2 + 2Ma/3}  \bigg).\]
\end{lemma}

Let $E^{ij}$ be the matrix with 1 in the $i,j$ and $j,i$ positions and 0 everywhere else. Let $p_{ij} = \popa_{ij}$. To use this inequality, express $\tilde C - \tilde\popl$ as the sum of the matrices $Y_{i,m+j}$,
\[Y_{i,m+j} = \frac{1}{\sqrt{(\popo_{ii}+{\tau})(\popp_{jj}+{\tau})}}(A_{ij}-p_{ij})E^{i, m+j}, i = 1,...,m, j = 1,...,n.\]
Note that 
\[\|\tilde C -\tilde \popl\| = \|\sum_{i = 1}^m\sum_{j = 1}^n Y_{i,m+j}\|,\]
and
 \[\|Y_{i,m+j}\| \leq \frac{1}{\sqrt{(\popo_{ii}+{\tau})(\popp_{jj}+{\tau})}} \leq (\delta+\tau)^{-1}.\]
Moreover,
\[\E [Y_{i,m+j}] = 0 \quad and \quad \E [Y_{i,m+j}^2] = \frac{1}{(\popo_{ii}+{\tau})(\popp_{jj}+{\tau})}(p_{ij} - p_{ij}^2)(E^{ii}+E^{m+j,m+j}).\]
Then,
\begin{align*}
v^2 &= \|\sum_{i = 1}^m\sum_{j = 1}^n \E [Y_{i,m+j}^2]\|  = \|\sum_{i = 1}^m\sum_{j = 1}^n \frac{1}{(\popo_{ii}+{\tau})(\popp_{jj}+{\tau})} (p_{ij} - p_{ij}^2)(E^{ii}+E^{m+j,m+j})\|\\
&= \|\sum_{i = 1}^m[\sum_{j = 1}^n \frac{1}{(\popo_{ii}+{\tau})(\popp_{jj}+{\tau})} (p_{ij} - p_{ij}^2)]E^{ii} + \sum_{j = 1}^n[\sum_{i = 1}^m \frac{1}{(\popo_{ii}+{\tau})(\popp_{jj}+{\tau})}(p_{ij} - p_{ij}^2)]E^{m+j,m+j}\|\\
&= \max \bigg\{ \max_{i = 1,...,m} (\sum_{j = 1}^n\frac{1}{(\popo_{ii}+{\tau})(\popp_{jj}+{\tau})} (p_{ij} - p_{ij}^2)), \max_{j = 1,...,n} (\sum_{i=1}^m\frac{1}{(\popo_{ii}+{\tau})(\popp_{jj}+{\tau})} (p_{ij} - p_{ij}^2))\bigg\}\\
&\leq \max \bigg\{ \max_{i = 1,...,m} \frac{1}{\delta + \tau}\sum_{j = 1}^n\frac{p_{ij}}{\popo_{ii}+ \tau}, \max_{j = 1,...,n} \frac{1}{\delta+\tau} \sum_{i = 1}^m\frac{p_{ij}}{\popp_{jj} + \tau}\bigg\}\\
&= (\delta+\tau)^{-1}.
\end{align*}
Take 
\[a = \sqrt{\frac{3 \ln (4(N_r+N_c)/\epsilon)}{\delta + \tau}}.\]
 By assumption, $\delta + \tau > 3\ln (N_r+N_c) + 3\ln(4/\epsilon)$. So $a < 1$. Applying Lemma~\ref{lem:61},
\begin{align*}
pr(\|\tilde C - \tilde\popl\| \ge a ) &\le 2(N_r+N_c) \exp \bigg(-\frac{\frac{3 \ln(4(N_r+N_c)/\epsilon)}{\delta + \tau}}{2/(\delta + \tau) + 2a/[3(\delta + \tau)]}\bigg)\\
&\le 2N \exp(-\frac{3 \ln (4(N_r+N_c)/\epsilon)}{3})\\
& \le \epsilon/2.
\end{align*}

For the second term $\|\tilde L - \tilde C\|$, define 
\[ D_\tau = \left( \begin{array}{cc}
   O_\tau & 0 \\
   0 & P_\tau \end{array} \right), \quad\quad \popd_\tau = \left( \begin{array}{cc}
   \popo_\tau & 0 \\
   0 & \popp_\tau \end{array} \right), \quad D = D_0, \ \mbox{ and }  \ \popd = \popd_0. \]

Apply the two sided concentration inequality for each $i$, $1\leq i\leq N_r+N_c$, (see for example \citet[chap. 2]{chung2006complex})
\[pr(|D_{ii} - \popd_{ii}| \ge \lambda) \le \exp\{-\frac{\lambda^2}{2\popd_{ii}}\} + \exp\{-\frac{\lambda^2}{2\popd_{ii} + \frac{2}{3}\lambda}\}.\]
Let $\lambda = a(\popd_{ii} + \tau)$, where $a$ is as before.
\begin{align*}
pr\bigg(|D_{ii} - \popd_{ii}| \ge a(\popd_{ii} + \tau)\bigg) &\le \exp\{-\frac{a^2(\popd_{ii} + \tau)^2}{2\popd_{ii}}\} + \exp\{-\frac{a^2(\popd_{ii} + \tau)^2}{2\popd_{ii} + \frac{2}{3}a(\popd_{ii} + \tau)}\}\\
& \le 2\exp\{-\frac{a^2(\popd_{ii} + \tau)^2}{(2+\frac{2}{3}a)(\popd_{ii} + \tau)}\} \\
&\le 2\exp\{-\frac{a^2(\popd_{ii} + \tau)}{3}\} \\
&\le 2\exp\{-\ln(4(N_r+N_c)/\epsilon)\frac{(\popd_{ii} + \tau)}{\delta + \tau}\} \\
&\le 2\exp\{-\ln(4(N_r+N_c)/\epsilon)\} \\
&\le \epsilon/2(N_r+N_c).
\end{align*}
Because
\begin{align*}
\|\popd_\tau^{-\frac{1}{2}} D_\tau^{\frac{1}{2}} - I \| = max_{i} \bigg|\sqrt{\frac{D_{ii}+ \tau}{\popd_{ii} + \tau}} - 1\bigg| \le max_{i} \bigg|\frac{D_{ii}+ \tau}{\popd_{ii} + \tau} - 1\bigg|,
\end{align*}
It follows that
\begin{align*}
pr(\|\popd_\tau^{-\frac{1}{2}}D_\tau^{\frac{1}{2}} - I \| \ge a) &\le pr(max_{i} \bigg|\frac{D_{ii}+ \tau}{\popd_{ii} + \tau} - 1\bigg| \ge a)\\
&\le pr(\cup_i \{|(D_{ii}+\tau) - (\popd_{ii}+\tau)| \ge a(\popd_{ii} + \tau)\}) \\
&\le \epsilon/2.
\end{align*}
Note that $\|\tilde L_\tau\| \le 1$. Therefore, with probability at least $1- \epsilon/2$,
\begin{align*}
\|\tilde L_\tau - C\| &= \|D_\tau^{-\frac{1}{2}}\tilde AD_{\tau}^{-\frac{1}{2}} - \popd_\tau^{-\frac{1}{2}}\tilde A \popd_{\tau}^{-\frac{1}{2}}\| \\
& = \|\tilde L_\tau - \popd_\tau^{-\frac{1}{2}}D_\tau^{\frac{1}{2}}\tilde L_\tau D_\tau^{\frac{1}{2}}\popd_\tau^{-\frac{1}{2}}\| \\
&= \|(I - \popd_\tau^{-\frac{1}{2}}D_\tau^{\frac{1}{2}})\tilde L_\tau D_\tau^{\frac{1}{2}}\popd_\tau^{-\frac{1}{2}} + \tilde L_\tau (I - D_\tau^{\frac{1}{2}}\popd_\tau^{-\frac{1}{2}})\| \\
&\le \|\popd_\tau^{-\frac{1}{2}}D_\tau^{\frac{1}{2}} - I \|\|\popd_\tau^{-\frac{1}{2}}D_\tau^{\frac{1}{2}} \| + \|\popd_\tau^{-\frac{1}{2}}D_\tau^{\frac{1}{2}} - I \| \\
\le a^2 + 2a.
\end{align*}

Combining the two parts yields 
\[ \|\tilde L_\tau - \tilde \popl_\tau\| \le a^2 + 3a \le 4a,\]
with probability at least $1- \epsilon$.
\end{proof}

 The next theorem bounds the difference between the empirical and population singular vectors 
 in terms of the Frobenius norm.
 
 \begin{theorem} \label{dk}{(Concentration of Singular Space)}
 Let A be the adjacency matrix generated from the DC-ScBM with parameters $\{{\bf B}, Y, Z, \Theta_Y, \Theta_Z\}$.
 Let $\lambda_1 \geq \lambda _2 \geq ... \geq \lambda_K >0$ be the positive singular values of $\popl_\tau$.
 
 Let $X_L (X_R)$ and $\popx_L (\popx_R)$ contain the top $K$ left(right) singular vectors of $L_\tau$ and $\popl_\tau$ respectively.  
For any $ \epsilon > 0$ and sufficiently large $N_r$ and $N_r$, 
if $\delta >3\ln(N_r+N_c) + 3\ln (4/\epsilon)$, 
then with probability at least  $1- \epsilon$
\begin{eqnarray}
\|X_L - \popx_L\rot_L\|_F &\leq& \frac{8\sqrt{6}}{\lambda_K} \sqrt{\frac{K\ln(4(N_r+N_c)/\epsilon)}{\delta+\tau}}\\
\mbox{ and } \  \|X_R - \popx_R\rot_R\|_F &\leq& \frac{8\sqrt{6}}{\lambda_K} \sqrt{\frac{K\ln(4(N_r+N_c)/\epsilon)}{\delta+\tau}},
  \end{eqnarray}
  for some orthogonal matrices $\rot_L, \rot_R \in \R^{K \times K}$.
  \end{theorem}

\begin{proof}
Define 
\[\tilde\popx = \frac{1}{\sqrt{2}}\left( \begin{array}{c}
   \popx_L \\
    \popx_R  \end{array} \right).\]    
A simple calculation shows that $\tilde\popx \in \R^{(N_r + N_c)\times K}$ contains the top $K$ eigenvectors of $\tilde L$ corresponding to its top $K$ eigenvalues. 

We apply an improved version of Davis Kahn theorem from \cite{lei2013consistency}. By a slightly modified proof of Lemma 5.1 in \citet{lei2013consistency}, 
it can be shown that 
\[\|\tilde X\tilde X^T - \tilde\popx\tilde\popx^T\|_F \le \frac{\sqrt{2K}}{\lambda_K}\|\tilde L_\tau - \tilde\popl_\tau\|.\]
Combining it with Theorem \ref{Concentration} and its assumptions, 
\[\|\tilde X\tilde X^T - \tilde\popx\tilde\popx^T\|_F \leq \frac{4\sqrt{6}}{\lambda_K} \sqrt{\frac{K\ln(4(N_r+N_c)/\epsilon)}{\delta+\tau}},\]
with probability at lease $1 - \epsilon$.
By definition of $\tilde \popx$ and $\tilde X$,
\begin{align*}
\|\tilde X\tilde X^T - \tilde\popx\tilde\popx^T\|_F &= \left\|\left( \begin{array}{cc}
   \frac{1}{2}(X_LX_L^T - \popx_L\popx_L^T) & \frac{1}{2}(X_LX_R^T - \popx_L\popx_R^T)\\
    \frac{1}{2}(X_RX_L^T - \popx_R\popx_L^T) & \frac{1}{2}(X_RX_R^T - \popx_R\popx_R)  \end{array} \right)\right\|_F \\
   & \ge \frac{1}{2}\|X_LX_L^T - \popx_L\popx_L^T\|_F\\
   & \ge \frac{1}{2}\|X_L - \popx_L\rot_L\|_F.
 \end{align*}
 Similarly $\|\tilde X\tilde X^T - \tilde\popx\tilde\popx^T\|_F \ge \frac{1}{2}\|X_R - \popx_R\rot_R\|_F$. This proves the above theorem.
\end{proof}

\section{Clustering} \label{clusterAppendix}

To rigorously discuss the asymptotic estimation properties of \textsc{di-sim}, the next subsections examine the behavior of \textsc{di-sim} applied to a population version of the graph Laplacian $\popl$, and compare this to \textsc{di-sim} applied to the observed graph Laplacian $L$.


 \subsection{The population version of \textsc{di-sim}} \label{populationalg}
 
  This subsection shows that \textsc{di-sim} applied to $\popl$ can perfectly identify the blocks in the Stochastic co-Blockmodel.  
Recall \textsc{di-sim} applied to $L$.
\begin{enumerate}
\item Find the left singular vectors $X_L \in \R^{N_r \times k_y}$.
\item Normalize  each row of $X_L$  to have unit length. Denote the normalized rows of $X_L$ as $u_1, \dots, u_{N_r} \in \R^{k_y}$ with and $\|u_i\|_2 = 1$.
\item Run ($1+\alpha$)-approximate $k$-means on $u_1, \dots, u_{N_r}$ with $k_y$ clusters.
\item Repeat steps (a), (b), and (c) for the the right singular vectors $X_R \in \R^{N_c \times k_y}$ with $k_z$ clusters. 
\end{enumerate}

$k$-means clusters points $u_1, \dots, u_n$ in Euclidean space by optimizing the following objective function (\cite{steinhaus1956division}),
\begin{equation}\label{kmeans}
\min_{\{m_1, \dots, m_{k_y}\} \subset \R^{k_y}} \sum_i \min_g \|u_i - m_g\|_2^2.
\end{equation}
Define the \textit{centroids} as the arguments $m_1^*, \dots, m_{k_y}^*$ that optimize (\ref{kmeans}). Finding $m_1^*, \dots, m_{k_y}^*$ is NP-hard. 
\textsc{di-sim} uses a linear time algorithm, ($1+\alpha$)-approximate $k$-means (\citet{kumar2004simple}). That is, the algorithm computes $\hat m_1, \dots, \hat m_{k_y}$ such that 
\[\sum_i \min_g \|u_i - \hat m_g\|_2^2 \le (1+\alpha)\sum_i \min_g \|u_i - m^*_g\|_2^2.\]

To study \textsc{di-sim} applied to $\popl$, Lemma \ref{lemPopL} gives an explicit form as a function of the parameters of the DC-ScBM.  
Recall that  $\popa = E(A)$ and under the DC-ScBM,
\[\popa = \Theta_yY{\bf B}Z^T\Theta_z,\]
where $Y\in \{0,1\}^{N_r \times k_y}, Z \in \{0,1\}^{N_c \times k_z},$ and $B \in [0,1]^{k_y \times k_z}$.  Assume that $k_y \le k_z$, without loss of generality.  
Moreover, recall that the regularized population versions of  $O$, $P$, and $L$ are defined as
\begin{equation}
    \begin{array}{lll}
\popp_{jj} = \sum_k \popa_{kj}  \B \\ 
\popo_{ii} = \sum_k \popa_{ik}  \B \\ 
\popo_\tau = \popo+\tau I, \quad\quad \popp_\tau = \popp+\tau I \B \\
\popl  =  \popo_\tau^{-\frac{1}{2}} \popa\popp_\tau^{-\frac{1}{2}} \B
\end{array}
\end{equation}
where $\popo_\tau$ and $\popp_\tau$ are diagonal matrices.

The following proves Lemma \ref{lemPopL}.
\begin{proof}
Define $O_B \in \R^{k_y \times k_y}$ as a diagonal matrix whose $(s,s)$'th element is $[O_B]_{ss} = \sum_t B_{st}$. Similarly define $P_B \in \R^{k_z \times k_z}$ as a diagonal matrix whose $(t,t)$'th element is $[P_B]_{tt} = \sum_s B_{st}$.
A couple lines of algebra shows that $[O_B]_{ss} $ is the total expected out-degrees of row nodes from block $s$ and that $\popo_{ii} = \theta^Y_{i}[O_B]_{y_i y_i}$.  Similarly $[P_B]_{tt} $ is the total expected in-degrees of column nodes from block $t$ and that $\popp_{jj} = \theta^Z_{j}[P_B]_{z_j z_j}$.

Recall that $\popo_{ii} = \theta^Y_i[P_B]_{y_i y_i}$ and $\popp_{jj} = \theta^Z_{j}[O_B]_{z_j z_j}$.
In addition, 
\[[\Theta_{Y,\tau}]_{ii} = \theta^Y_i\frac{\popo_{ii}}{\popo_{ii} + \tau} \quad \mbox{ and } \quad [\Theta_{Z,\tau}]_{jj} = \theta^Z_j\frac{\popp_{jj}}{\popp_{jj} + \tau}.\]

The $ij$'th element of $\popl_\tau$ is
\[[\popl]_{ij} = \frac{\popa_{ij}}{\sqrt{(\popo_{ii} + \tau)(\popp_{jj} + \tau)}} = \frac{\theta^Y_i\theta^Z_jB_{y_iz_j}}{\sqrt{\popo_{ii}\popp_{jj}}}\sqrt{\frac{\popo_{ii}}{\popo_{ii}+\tau}\frac{\popp_{jj}}{\popp_{jj}+\tau}} = \frac{B_{z_iz_j}}{\sqrt{[P_B]_{y_i}[O_B]_{z_j}}}\sqrt{[\Theta_{Y,\tau}]_{ii}[\Theta_{Z,\tau}]_{jj}}.\]
Hence, 
 \[\popl = \Theta_{Y,\tau}^{\frac{1}{2}}ZB_LZ^T\Theta_{Z,\tau}^{\frac{1}{2}},\]
where $B_L$ is defined  as
\begin{equation} \label{def:bl}
B_L = O_B^{-1/2}BP_B^{-1/2}.
\end{equation}
 \end{proof}

Recall that $\popa = \Theta_Y Y {\bf B} Z^T \Theta_Z$.  Lemma \ref{lemPopL} demonstrates that $\popl$ has a similarly simple form that separates the block-related information ($B_L$) and node specific information ($\Theta_Y$ and $\Theta_Z$). 

Assume that $rank(B_L) = K, 0 < K =k_y \le k_z$.  Recall $H = (Y^T\Theta_{Y,\tau} Y)^{\frac{1}{2}}B_L(Z^T\Theta_{Z,\tau}Z)^{\frac{1}{2}}$.
Singular value decomposition of H gives 
\[H = U\Lambda V^T.\]
where $U \in \R^{k_y \times K}/V \in \R^{k_z \times K}$ is the left/right singular vector of $H$ and $\Lambda \in \R^{K \times K}$ is diagonal containing the positive singular values of $H$, $\lambda_1\ge \lambda_2 \ge... \ge \lambda_K >0$.
The proof of the next lemma shows that $H$ and $\popl$ share the same nonzero singular values.

The next lemma gives the explicit form of the left and right population singular vectors and further shows that their normalized versions are block constant.
 \begin{lemma} \label{lemsvd} (Singular value decomposition for $\popl$) Under the DC-ScBM with parameters $\{{\bf B}, Y, Z, \Theta_Y, \Theta_Z\}$, 
 Let $\popx_L \in \R^{N_r \times K}(\popx_R \in \R^{N_c \times K})$ contain the left/right singular vectors of $\popl_\tau$. Define $\popx_L^*/\popx_R^*$to be the row-normalized $\popx_L/\popx_R$. Then 
\begin{enumerate}
\item $\popx_{L} = \Theta_{Y,\tau}^{\frac{1}{2}}Y(Y^T\Theta_{Y,\tau} Y)^{-\frac{1}{2}}U$,
\item $\popx_{R} = \Theta_{Z,\tau}^{\frac{1}{2}}Z(Z^T\Theta_{Z,\tau} Z)^{-\frac{1}{2}}V$.
\item $\popx_L^* =  YU$, $Y_i \neq Y_j \Leftrightarrow Y_iU \neq Y_jU$.
\item $\popx_R^* = ZV^*$, where $V^*_j = V_j/\|V_j\|_2$.
\end{enumerate}
 \end{lemma}

 \begin{proof}
Recall that  $H = (Y^T\Theta_{Y,\tau} Y)^{\frac{1}{2}}B_L(Z^T\Theta_{Z,\tau }Z)^{\frac{1}{2}}$ and singular value decompositon of H gives $H = U\Lambda V^T$.
 
 Define $\popx_{L} = \Theta_{Y,\tau}^{\frac{1}{2}}Y(Y^T\Theta_{Y,\tau} Y)^{-\frac{1}{2}}U$, and 
$\popx_{R} = \Theta_{Z,\tau}^{\frac{1}{2}}Z(Z^T\Theta_{Z,\tau} Z)^{-\frac{1}{2}}V$.
It is easy to check that $\popx_{L}^T\popx_{L} = I$ and $\popx_{R}^T\popx_{R} = I$.

 On the other hand, 
 \begin{align*}
 \popx_{L} \Lambda \popx_{R}^T = \Theta_{Y,\tau}^{\frac{1}{2}}YB_LZ^T\Theta_{Z,\tau}^{\frac{1}{2}} = \popl.
 \end{align*}
 Hence, $\lambda_s, s = 1,...,r$ are $\popl_\tau$'s nonzero singular values and $\popx_{L}/\popx_{R}$ contains $\popl_\tau$'s left/right singular vectors corresponding to its nonzero singular values.
 
Let $\popx_{L}^i$ denote the $i$'th row of $\popx_{L}$. For part (c), notice that  
 \[\|\popx_{L}^i\|_2 = (\frac{[\Theta_{Y,\tau}]_{ii}}{[Y^T\Theta_{Y,\tau} Y]_{y_iy_i}})^{\frac{1}{2}}.\]
 So,
\[ [\popx_{L}^*]^i = \frac{\popx_{L}^i}{\|\popx_{L}^i\|_2}  = Y_iU. \]
 Therefore, $\popx_{L}^* = YU$. 
 For (d), notice that 
 \[\|\popx_{R}^j\|_2 = (\frac{[\Theta_{Z,\tau}]_{jj}\|V_{Z_j}\|^2}{[Z^T\Theta_{Z,\tau} Z]_{z_jz_j}})^{\frac{1}{2}}. \]
  Hence, 
  \[ [\popx_{R}^*]^j = \frac{\popx_{R}^j}{\|\popx_{R}^j\|_2}  = Z_jV^*. \]
  \end{proof}

\subsection{Comparing the population and observed clusters} \label{comparing}
The first part of the section proves the bound of misclustering rate for row nodes.
\subsubsection{Clustering for $Y$}
\begin{proof}
Recall that the set of misclustered row nodes is defined as:
\begin{align*}
\mathscr{M}_y = \left\{i : \|c_i^L- y_i \mu^y\rot_L\|_2 > \|c_i^L - y_j \mu^y\rot_L\|_2 \ \mbox{ for any } \ y_j \ne y_i\right\}.
\end{align*}
Let $\mathcal{C}_i$ denote  $y_i \mu^y$.
Note that Lemma \ref{lemsvd} implies that the population centroid corresponding to the $i$'th row of $\popx^*_L$ is
 \[ \mathcal{C}_i = y_i \mu^y= y_iU.\]  
 Since all population centroids are of unit length and are orthogonal to each other, a simple calculation gives a sufficient condition for one observed centroid to be closest to the population centroid:
\begin{align*}
\|c^L_i \rot_L^T - \mathcal{C}^L_i\|_2 < 1/\sqrt{2} \Rightarrow \|c_i^L\rot_L^T - \mathcal{C}^L_i\|_2 < \|c^L_i\rot_L^T - \mathcal{C}^L_j\|_2, \quad \forall j \neq i.
\end{align*}
Define the following set of nodes that do not satisfy the sufficient condition,
\begin{align*}
\mathscr{B}_y = \{i: \|c^L_i\rot_L^T - \mathcal{C}^L_i\|_2 \ge 1/\sqrt{2}\}.
\end{align*}
The mis-clustered nodes $\mathscr{M}_y \subset \mathscr{B}_y$. 

Define $C_L \in \R^{N_r \times K}$, where the $i$'th row of $C_L$ is $c^L_i$, the observed centroid of node $i$ from the $(1+\alpha)$-approximate k-means. 
Define $M_L \in \R^{N_r \times K}$ to be the global solution of k-means.
By definition, 
\[  \|X^*_L - C_L\|_F \le(1+\alpha)\|X^*_L - M_L\|_F \le (1+\alpha)\|X^*_L - \popx^*_L\rot_L\|_F.\]
Further, by the triangle inequality,
\[\|C_L - YU\rot_L\|_F = \|C_L - \popx^*_L\rot_L\|_F \le \|X^*_L - C_L\|_F + \|X^*_L - \popx^*_L\rot_L\|_F \le (2+\alpha)  \|X^*_L - \popx^*_L\rot_L\|_F.\]
Thus,
\begin{align*}
\frac{|\mathscr{M}_y|}{N_r} &\le \frac{|\mathscr{B}_y|}{N_r} = \frac{1}{N_r}\sum_{i \in \mathscr{B}_y} 1 \\
&\le \frac{2}{N_r} \sum_{i \in \mathscr{B}_y}\|c^L_i\rot_L^T - \mathcal{C}^L_i\|_2^2\\
&= \frac{2}{N_r}\|C_L - YU\rot_L\|_F^2\\
&\le \frac{2(2+\alpha)^2}{N_r}\|X^*_L - \popx^*_L\rot_L\|_F^2\\
&\le \frac{8(2+\alpha)^2}{N_rm_y^2}\|X_L - \popx_L\rot_L\|_F^2.
\end{align*}

The last inequality is due to the following fact.
\begin{lemma}\label{simple}
For two non-zero vectors $v_1, v_2$ of the same dimension, we have 
\[\|\frac{v_1}{\|v_1\|_2} - \frac{v_2}{\|v_2\|_2}\|_2 \le 2\frac{\|v_1-v_2\|_2}{max(\|v_1\|_2, \|v_2\|_2)}.\]
\end{lemma}

By Theorem \ref{dk}, we have, with probability at least $1 - \epsilon$,
\[\frac{|\mathscr{M}_y|}{N_r} \le c_0(\alpha)\frac{K\ln(4(N_r+N_c)/\epsilon)}{N_r\lambda_K^2m_y^2(\delta + \tau)}.\]
\end{proof}

The second part proves the bound of the misclustering rate for column nodes.
\subsubsection{Clustering for $Z$}
Because $k_y \le k_z$, it is slightly more challenging to bound $\mathscr{M}_z$.

\begin{proof}
Recall that  $H = (Y^T\Theta_{Y,\tau} Y)^{\frac{1}{2}}B_L(Z^T\Theta_{Z,\tau}Z)^{\frac{1}{2}}$ and $H = U\Lambda V^T$.
Left multiply by $\Lambda^{-1}U^T$, we have
\[V = H^TU\Lambda^{-1}.\]
Hence 
\[\|V_i - V_j\|_2 \ge \frac{1}{\lambda_1} \|H_{\cdot i}U - H_{\cdot j}U\|_2 \ge \|H_{\cdot i} - H_{\cdot j}\|_2.\]
The second inequality is due to the facts that $\lambda_1\le  1$ and $U$ is an orthogonal matrix.
Recall that 
\[\gamma_z = \min_{i \neq j}\|H_{\cdot i} - H_{\cdot j}\|_2 +(1- \kappa),\]
where $\kappa = \max_{i,j} \|V_i\|_2/\|V_j\|_2$.
We have that, $\forall i\neq j$,
\[\|V^*_i - V^*_j\|_2 \ge \gamma_z.\]
This is because
\begin{align*}
\|V^*_i - V^*_j\|_2 &= \|\frac{V_i - V_j}{\|V_j\|_2} + V_i(\frac{1}{\|V_i\|_2} - \frac{1}{\|V_j\|_2})\|_2\\
&\ge \|V_i - V_j\|_2 + 1 - \frac{\|V_i\|_2}{\|V_j\|_2}\\
&\ge \|H_{\cdot i} - H_{\cdot j}\|_2 +(1- \kappa)\\
&\ge \gamma_z.
\end{align*}

Recall that the set of misclustered row nodes is defined as:
\begin{align*}
\mathscr{M}_z = \left\{i : \|c_i^R- z_i \mu^z\rot_R\|_2 > \|c_i^R - z_j \mu^z\rot_R\|_2 \ \mbox{ for any } \ z_j \ne z_i\right\}.
\end{align*}
Let $\mathcal{C}^R_i$ denote $z_i \mu^z$. Note that Lemma \ref{lemsvd} implies that the population centroid corresponding to the $i$'th row of $\popx^*_R$ is
\[ \mathcal{C}^R_i = z_i \mu^z = Z_iV^*.\]  

Define the following set of column nodes,
\begin{align*}
\mathscr{B}_z = \{i: \|c^R_i\rot_R^T - \mathcal{C}^R_i\|_2 \ge \gamma_z/2\}.
\end{align*}
It is straightforward to show that $\mathscr{M}_z \in \mathscr{B}_z$. 

Define $C_R \in \R^{N_c \times K}$, where the $i$'th row of M is $c^R_i$, the observed centroid of column node $i$ from $(1+\alpha)$-approximate k-means.
Define $M_R \in \R^{N_r \times K}$ to be the global solution of k-means.
 By definition, we have 
\[\|X^*_R - C_R\|_F \le (1+\alpha)\|X^*_R - M_R\|_F\le(1+\alpha)\|X^*_R - \popx^*_R\rot_R\|_F.\]
Further, by the triangle inequality,
\[\|C_R - ZV^*\rot_R\|_F = \|C_R - \popx^*_R\rot_R\|_F \le \|X^*_R - C_R\|_F + \|X^*_R - \popx^*_R\rot_R\|_F \le (2+\alpha)  \|X^*_R - \popx^*_R\rot_R\|_F.\]
Putting all of these pieces together,
\begin{align*}
\frac{|\mathscr{M}_z|}{N_c} &\le \frac{|\mathscr{B}_z|}{N_c} = \frac{1}{N_c}\sum_{i \in \mathscr{B}_z} 1 \\
&\le \frac{4}{N_c\gamma_z^2} \sum_{i \in \mathscr{B}_y}\|c^R_i\rot_L^R - \mathcal{C}^R_i\|_2^2\\
&= \frac{4}{N_c\gamma_z^2}\|C_R - ZV^*\rot_R\|_F^2\\
&\le \frac{4(2+\alpha)^2}{N_c\gamma_z^2}\|X^*_R- \popx^*_R\rot_R\|_F^2\\
&\le \frac{16(2+\alpha)^2}{N_c\gamma_z^2m_z^2}\|X_R- \popx_R\rot_R\|_F^2.
\end{align*}

By Theorem \ref{dk}, we have with probability at least $1 - \epsilon$,
\[\frac{|\mathscr{M}_z|}{N_c} \le c_1(\alpha)\frac{K\ln(4(N_r+N_c))/\epsilon)}{N_r\lambda_K^2m_z^2\gamma_z^2(\delta + \tau)}.\]
\end{proof}

The following is a proof of Corollary \ref{fourcor}.
\begin{proof}
Under the four parameter ScBM, presume that $\theta_i=1/s$ for all $i$.  From the proof of Lemma \ref{lemsvd},  $\popl$ has the same singular values as  
\[H = (Y^T\Theta_{Y,\tau=0} Y)^{\frac{1}{2}}B_L(Z^T\Theta_{Z,\tau=0}Z)^{\frac{1}{2}} = B_L =  O_B^{-\frac{1}{2}}BP_B^{-\frac{1}{2}} = \frac{1}{s^2(Kr + p)} (s^2pI_K + s^2r\textbf{1}_K\textbf{1}_K^T).\]
By inspection, the constant vector is an eigenvector of this matrix.  It has eigenvalue
\[\lambda_1 =  \frac{p + Kr}{Kr + p} = 1.\]
Any vector orthogonal to a constant vector is also an eigenvector.  These eigenvectors have eigenvalue 
\[\lambda_k = \frac{p}{Kr + p} = \frac{1}{K(r/p) + 1}.\]
The result follows from using $m_y^2 = K/n$ (see discussion after Theorem \ref{clusterTheorem}) and $\delta \propto N$.
\end{proof}

\end{document}